%% file: sample_paper.tex
\documentclass[twoside]{article}

\usepackage[accepted]{aistats2024}
\usepackage[round]{natbib}

\usepackage{hyperref}       
\usepackage{url}            
\usepackage{booktabs}       
\usepackage{amsfonts}       
\usepackage{nicefrac}       
\usepackage{microtype}      
\usepackage{graphicx}
\usepackage{doi}
\usepackage{amsmath}
\usepackage{amssymb}
\usepackage{mathtools}
\usepackage{amsthm}
\usepackage{minitoc}
\usepackage{footmisc}

\input{inouye-header} 
\input{inouye-macros} 

\theoremstyle{plain}
\newtheorem{theorem}{Theorem}
\newtheorem{proposition}[theorem]{Proposition}
\newtheorem{lemma}[theorem]{Lemma}
\newtheorem{corollary}[theorem]{Corollary}
\theoremstyle{definition}
\newtheorem{definition}{Definition}

\theoremstyle{remark}

\definecolor{Maroon}{RGB}{128, 0, 0}
\definecolor{Violet}{RGB}{127, 0, 255}
\definecolor{PineGreen}{RGB}{1, 121, 111}
\definecolor{RedOrange}{RGB}{100, 33, 29}
\definecolor{BrickRed}{RGB}{203, 65, 84}
\definecolor{BlueViolet}{RGB}{138, 43, 226}

\newcommand{\changeziyu}[1]{{#1}}

\begin{document}

%

%

\twocolumn[

\aistatstitle{Towards Practical Non-Adversarial Distribution Matching}

\aistatsauthor{ Ziyu Gong\And Ben Usman \And Han Zhao}

\aistatsaddress{ Purdue University \\ \texttt{gong123@purdue.edu} \And  Google Research \\\texttt{usmn@google.com} \And University of Illinois Urbana-Champaign\\\texttt{hanzhao@illinois.edu}}

\aistatsauthor{ David I. Inouye}
\aistatsaddress{Purdue University \\ \texttt{dinouye@purdue.edu}} 

]

\renewcommand{\paragraph}[1]{\textbf{#1.} }
\input{template-aistats-matching}

\clearpage
\section*{Acknowledgement}
Z.G. and D.I. acknowledge support from NSF (IIS-2212097) and ARL (W911NF-2020-221). H.Z. was partly supported by the Defense Advanced Research Projects Agency (DARPA) under Cooperative Agreement Number: HR00112320012, an IBM-IL Discovery Accelerator Institute research award, and Amazon AWS Cloud Credit.
\bibliography{references}

\section*{Checklist}

The checklist follows the references. For each question, choose your answer from the three possible options: Yes, No, Not Applicable.  You are encouraged to include a justification to your answer, either by referencing the appropriate section of your paper or providing a brief inline description (1-2 sentences). 
Please do not modify the questions.  Note that the Checklist section does not count towards the page limit. Not including the checklist in the first submission won't result in desk rejection, although in such case we will ask you to upload it during the author response period and include it in camera ready (if accepted).

\textbf{In your paper, please delete this instructions block and only keep the Checklist section heading above along with the questions/answers below.}

 \begin{enumerate}

 \item For all models and algorithms presented, check if you include:
 \begin{enumerate}
   \item A clear description of the mathematical setting, assumptions, algorithm, and/or model.
   Yes
   \item An analysis of the properties and complexity (time, space, sample size) of any algorithm. 
   Not Applicable
   \item (Optional) Anonymized source code, with specification of all dependencies, including external libraries. 
   Yes
 \end{enumerate}

 \item For any theoretical claim, check if you include:
 \begin{enumerate}
   \item Statements of the full set of assumptions of all theoretical results.
   Yes
   \item Complete proofs of all theoretical results.
   Yes
   \item Clear explanations of any assumptions.
   Yes
 \end{enumerate}

 \item For all figures and tables that present empirical results, check if you include:
 \begin{enumerate}
   \item The code, data, and instructions needed to reproduce the main experimental results (either in the supplemental material or as a URL).
   Yes
   \item All the training details (e.g., data splits, hyperparameters, how they were chosen).
   Yes
   \item A clear definition of the specific measure or statistics and error bars (e.g., with respect to the random seed after running experiments multiple times).
   Yes
   \item A description of the computing infrastructure used. (e.g., type of GPUs, internal cluster, or cloud provider).
   Yes
 \end{enumerate}

 \item If you are using existing assets (e.g., code, data, models) or curating/releasing new assets, check if you include:
 \begin{enumerate}
   \item Citations of the creator If your work uses existing assets.
   Yes
   \item The license information of the assets, if applicable.
   Not Applicable
   \item New assets either in the supplemental material or as a URL, if applicable.
   Not Applicable
   \item Information about consent from data providers/curators.
   Not Applicable
   \item Discussion of sensible content if applicable, e.g., personally identifiable information or offensive content.
   Not Applicable
 \end{enumerate}

 \item If you used crowdsourcing or conducted research with human subjects, check if you include:
 \begin{enumerate}
   \item The full text of instructions given to participants and screenshots.
   Not Applicable
   \item Descriptions of potential participant risks, with links to Institutional Review Board (IRB) approvals if applicable.
   Not Applicable
   \item The estimated hourly wage paid to participants and the total amount spent on participant compensation.
   Not Applicable
 \end{enumerate}

 \end{enumerate}
 
\clearpage
\onecolumn
\appendix
\renewcommand{\paragraph}[1]{\textbf{#1} }
\input{appendix-aistats}

\end{document}

%% file: inouye-header.tex
\usepackage{wrapfig}
\usepackage{caption}
\usepackage{subcaption}
\usepackage{bbm}

\usepackage{graphicx}

\usepackage{url}
\usepackage{hyperref} 
\usepackage{natbib}

\usepackage{algorithm} 
\usepackage{algorithmic} 

\usepackage{mathtools} 
\usepackage{amsthm, amssymb, amscd, amsfonts} 
\usepackage{thmtools}
\usepackage{thm-restate} 
\usepackage{bm} 

\usepackage{setspace}
\usepackage{multirow}
\usepackage{enumitem} 
\usepackage[normalem]{ulem}
\usepackage{xspace} 
\usepackage{comment}
\usepackage{array}

%% file: inouye-macros.tex
\DeclareMathOperator*{\argmin}{arg\,min\,}
\DeclareMathOperator*{\argmax}{arg\,max\,}

\newcommand{\x}{x}
\newcommand{\xvec}{\bm{\x}}
\newcommand{\xmat}{X}
\newcommand{\xset}{\mathcal{\xmat}}

\newcommand{\z}{z}
\newcommand{\zvec}{\bm{\z}}

\newcommand{\normal}{\mathcal{N}}

\newcommand{\E}{\mathbb{E}}

\newcommand{\g}{g}

\renewcommand{\d}{d}

\newcommand{\pset}{\mathcal{P}}

\DeclareMathOperator{\JSD}{JSD}
\DeclareMathOperator{\GJSD}{GJSD}

\DeclareMathOperator{\KL}{KL}

\DeclareMathOperator{\HH}{H}
\DeclareMathOperator{\Hc}{H_c}

\DeclareMathOperator{\mname}{AUB}

\newcommand{\ndist}{k}

\newcommand{\sumd}{{\textstyle \sum_\d}}

\usepackage[many]{tcolorbox}
\newtcolorbox{outcomes}[1][]{
  breakable,
  title=#1,
  colback=white,
  colbacktitle=white,
  coltitle=black,
  fonttitle=\bfseries,
  bottomrule=3pt,
  toprule=3pt,
  leftrule=0pt,
  rightrule=0pt,
  titlerule=0pt,
  arc=0pt,
  outer arc=0pt,
  colframe=black,
}

%% file: template-aistats-matching.tex
\begin{abstract}
Distribution matching can be used to learn invariant representations with applications in fairness and robustness.
Most prior works resort to adversarial matching methods 
but the resulting minimax problems are unstable and challenging to optimize.
Non-adversarial likelihood-based approaches either require model invertibility, impose constraints on the latent prior, or lack a generic framework for distribution matching.
To overcome these limitations, we propose a non-adversarial VAE-based matching method that can be applied to any model pipeline.
We develop a set of alignment upper bounds for distribution matching (including a noisy bound) that have VAE-like objectives but with a different perspective.
We carefully compare our method to prior VAE-based matching approaches both theoretically and empirically.
Finally, we demonstrate that our novel matching losses can replace adversarial losses in standard invariant representation learning pipelines without modifying the original architectures---thereby significantly broadening the applicability of non-adversarial matching methods.
\end{abstract}

\section{INTRODUCTION}\label{sec:intro}

Distribution matching can be used to learn invariant representations that have many applications in robustness, fairness, and causality.
For example, in Domain Adversarial Neural Networks (DANN) \citep{ganin2016domain,zhao2018adversarial}, the key is enforcing an intermediate latent space to be invariant with respect to the domain.
In fair representation learning (e.g., \citet{creager2019flexibly, louizos2015variational, xu2018fairgan,Zhao2020Conditional}), a common approach is to enforce that a latent representation is invariant with respect to a sensitive attribute.
In both of these cases, distribution matching is formulated as a (soft) constraint or regularization on the overall problem that is motivated by the context (either domain adaptation or fairness constraints).
Thus, there is an ever-increasing need for reliable distribution matching methods.

Most prior works of distribution matching resort to adversarial training to implement the required matching constraints.
While adversarial loss terms are easy to implement as they only require a discriminator network, the corresponding minimax optimization problems are unstable and difficult to optimize in practice (see e.g. \citet{lucic2018gans,kurach2019the,pmlr-v119-farnia20a,nie2020towards,NEURIPS2020_3eb46aa5,han2023ffb}) in part because of the competitive nature of the min-max optimization problem.
To reduce the dependence on adversarial learning, \citet{grover2020alignflow} proposed an invertible flow-based method to combine likelihood and adversarial losses under a common framework.
\cite{usman2020log} proposed a completely non-adversarial matching method using invertible flow-based models where one distribution is assumed to be fixed.
\cite{cho2022cooperative} unified these previous non-adversarial flow-based approaches for distribution matching by proving that they are upper bounds of the Jensen-Shannon divergence called the alignment upper bound (AUB).
However, these non-adversarial methods require invertible model pipelines, which significantly limit their applicability in key distribution matching applications.
For example, because invertibility is required, the aligner cannot reduce the dimensionality.
As another consequence, it is impossible to use a shared invertible aligner because JSD is invariant under invertible transformations\changeziyu{\footnote{refer to \autoref{appendix:proof-JSD-invar} for proof}}.
Most importantly, invertible architectures are difficult to design and optimize compared to general neural networks.
Finally, several fairness-oriented works have proposed variational upper bounds \citep{louizos2015variational,moyer2018invariant,gupta2021controllable} for distribution matching via the perspective of mutual information. Some terms in these bounds can be seen as special cases of our proposed bounds (detailed in \autoref{sec:revisiting-vae}; also see comparisons to these methods in \autoref{sec:experiments-comparison-other-bounds}).
However, these prior works impose a fixed prior distribution and are focused on the fairness application, i.e., they do not explore the broader applicability of their matching bounds beyond fairness.

To address these issues, we propose a VAE-based matching method that can be applied to any model pipeline, i.e., it is model-agnostic like adversarial methods, but it is also non-adversarial, i.e., it forms a min-min cooperative problem instead of a min-max problem.
Inspired by the flow-based alignment upper bound (AUB)\citep{cho2022cooperative}, we relax the invertibility of AUB by replacing the flow with a VAE and add a mutual information term that simplifies to a $\beta$-VAE \citep{higgins2017betavae} matching formulation.
From another perspective, our development can be seen as revisiting VAE-based matching bounds where we show that prior works impose an unnecessary constraint caused by using a fixed prior and do not encourage information preservation as our proposed relaxation of AUB.
We then prove novel noisy alignment upper bounds as summarized in \autoref{tab:alignment-bounds}, which may help avoid vanishing gradient and local minimum issues that may exist when using the standard JSD as a divergence measure~\citep{arjovsky2017wasserstein}.\footnote{while the term \textit{alignment} in deep learning is often being interpreted as ``bringing human values and goals into models", for the purposes of maintaining consistency with the terminology employed in the AUB paper, we instead use the term \textit{alignment} interchangeably with the idea of \textit{distribution matching} throughout the rest of the paper.}
This JSD perspective to distribution matching complements and enhances the mutual information perspective in prior works \citep{moyer2018invariant, gupta2021controllable} because the well-known equivalence between mutual information of the observed features and the domain label and the JSD between the domain distributions.
Our contributions can be summarized as follows: 
\begin{itemize}[wide=6pt,leftmargin=*,topsep=0pt,itemsep=-1ex,partopsep=1ex,parsep=1ex]
    \item We relax the invertibility constraint of AUB using VAEs and a mutual information term while ensuring the distribution matching loss is an upper bound of the JSD up to a constant.
    \item We propose noisy JSD and noise-smoothed JSD to help avoid vanishing gradients and local minima during optimization, and we develop the corresponding noisy alignment upper bounds.
    \item We demonstrate that our non-adversarial VAE-based matching losses can replace adversarial losses without any change to the original model's architecture.
    Thus, they can be used within any standard invariant representation learning pipeline such as domain-adversarial neural networks or fair representation learning without modifying the original architectures.
\end{itemize}

\paragraph{Notation} Let $\xvec$ and $\d \in \{1,2,\dots,k\}$ denote the observed variable and the domain label, respectively, where $k > 1$ is the number of domains. Let $\zvec = \g(\xvec| \d)$ denote a deterministic representation function, i.e., an \emph{aligner}, that is invertible w.r.t.\ $\xvec$ conditioned on $\d$ being known.
Let $q$ denote the \emph{encoder} distribution: $q(\xvec,\d,\zvec) = q(\xvec, \d)q(\zvec|\xvec,\d)$, where $q(\xvec,\d)$ is the true data distribution and $q(\zvec|\xvec,\d)$ is the encoder (i.e., probabilistic aligner).
Similarly, let $p$ denote the \emph{decoder} distribution $p(\xvec,\d,\zvec)=p(\zvec)p(\d)p(\xvec|\zvec,\d)$, where $p(\zvec)$ is the shared prior, $p(\xvec|\zvec,\d)$ is the decoder, and $p(\d)=q(\d)$ is the marginal distribution of the domain labels.
Entropy and cross entropy will be denoted as $\HH(\cdot)$, and $\Hc(\cdot,\cdot)$, respectively, where the KL divergence is denoted as $\KL(p,q) = \Hc(p,q) - \HH(p)$.
Jensen-Shannon Divergence (JSD) is denoted as $\JSD(p,q) = \HH(\frac{1}{2}(p+q)) - \frac{1}{2}(\HH(p) + \HH(q))$.
Furthermore, the Generalized JSD (GJSD) extends JSD to multiple distributions \citep{lin1991divergence} and is equivalent to the mutual information between $\zvec$ and $\d$: $\GJSD(\{q(\zvec|\d)\}_{\d=1}^\ndist) \equiv I(\zvec, \d) = \HH(\E_{q(\d)}[q(\zvec|\d)]) - \E_{q(\d)}[\HH(q(\zvec|\d))]$, where $q(\d)$ are the weights for each domain distribution $q(\zvec|\d)$. JSD is recovered if there are two domains and $q(\d)=\frac{1}{2}, \forall d$.

\begin{table*}[!ht]

    \centering
    \caption{Summary of new alignment upper bounds where $C\triangleq -\E_{q(d)}[\HH(q(\xvec|\d))]$ and Flow+JSD is prior work from \cite{cho2022cooperative}. Blue represents changes needed for VAEs, i.e., replacing log Jacobian term and adding $\beta$ term to encourage near-invertibility of encoder, and red highlights that we merely need to add noise to the latent representation before evaluating the shared latent distribution. JSD bounds can be made tight via optimization while noisy JSD bounds can only be made tight if the noise variance goes to 0. 
    }
    \label{tab:alignment-bounds}
    \begin{tabular}{m{2.5cm}|m{6.5cm}|m{6.5cm}}
        Model
        & 
        JSD 
        &
        Noisy JSD 
        \\ \hline
        $\begin{aligned}
        &\textnormal{\textbf{Flow}} \\
        &\zvec=g(\xvec|d)
        \end{aligned}$
        & 
        $\begin{aligned}
        \,\min_{p(\zvec)}\,\,\, \E_{q}[-\log (|J_g(\xvec|d)|\cdot p(\zvec))] \!+\! C
        \end{aligned}$
        &
        $\begin{aligned}
        \,\min_{p(\tilde{\zvec})} \,\,\,\E_q[-\log (|J_g(\xvec|d)| \cdot p(\zvec \textcolor{BrickRed}{+ \epsilon}))] \!+\!C
        \end{aligned}$
        \\ [20pt]
        $\begin{aligned}
        &\textnormal{\textbf{$\beta$-VAE}\, ($\beta \leq 1$)} \\
        &\zvec \sim q(\zvec|\xvec,\d)
        \end{aligned}$
        & 
        $\begin{aligned}
        \!\!\min_{\substack{p(\zvec) \\ \textcolor{BlueViolet}{p(\xvec|\zvec,\d)}}} \E_q\!\!\left[-\log \left(\textcolor{BlueViolet}{\frac{p(\xvec|\zvec, \d)}{q(\zvec |\xvec, \d)^\beta}} \cdot p(\zvec)^{\textcolor{BlueViolet}{\beta}}\right) \right]\!+\! C
        \end{aligned}$
        &
        $\begin{aligned}
        \!\!\min_{\substack{p(\tilde{\zvec}) \\ 
        \textcolor{BlueViolet}{p(\xvec|\zvec,\d)}}} 
        \E_{q}\!\!\left[-\log \left(\textcolor{BlueViolet}{\frac{p(\xvec|\zvec, \d)}{q(\zvec |\xvec, \d)^\beta}} \cdot p(\zvec \textcolor{BrickRed}{+ \epsilon})^{\textcolor{BlueViolet}{\beta}}\right)\right] \!+\! C \\
        \end{aligned}$
        \\ [20pt]
    \end{tabular}
    \vspace{-1em}

\end{table*}

\section{BACKGROUND}

\paragraph{Adversarial Methods} 
Adversarial methods based on GANs \citep{goodfellow2014generative} maximize a \emph{lower bound} on the GJSD using a probabilistic classifier denoted by $f$:
\vspace{0.5em}
\begin{equation}
\begin{aligned}
&\min_\g \left(\max_f \E_{q(\xvec,d)}[\log f_d(\g(\xvec|d))] \right) \notag \\
&=\min_\g \textnormal{ADV}(\g) \leq \min_\g \GJSD(\{q(z|d)\}_{d=1}^k).
\label{eqn:adversarial-loss}
\end{aligned}
\vspace{0.5em}
\end{equation}
where $\textnormal{ADV}(\g)$ is the adversarial loss that lower bounds the GJSD and can be made tight if $f$ is optimized overall all possible classifiers. 
This optimization can be difficult to optimize in practice due to its adversarial formulation and vanishing gradients caused by JSD (see e.g. \cite{lucic2018gans,kurach2019the,pmlr-v119-farnia20a,nie2020towards,NEURIPS2020_3eb46aa5,han2023ffb,arjovsky2017towards}).
We aim to address both optimization issues by forming a min-min problem (\autoref{sec:relaxing-invertibility}) and considering additive noise to avoid vanishing gradients (\autoref{sec:noisy-jsd}).

\paragraph{Fair VAE Methods}
A series of prior works in fairness implemented distribution matching methods based on VAEs \citep{kingma2019introduction}, where the prior is assumed to be the standard normal distribution $\mathcal{N}(0,I)$ and the probabilistic encoder represents a stochastic aligner.
Concretely, the fair VAE objective \citep{louizos2015variational} can be viewed as an \emph{upper bound} on the GJSD:
\vspace{0.5em}
\begin{equation}
    \begin{aligned}
    &\min_{q(\zvec|\xvec,\d)} \!\!\big(\min_{p(\xvec|\zvec,\d)} 
    \E_{q(\xvec,\zvec,\d)}\big[-\log \frac{p(\xvec|\zvec, \d)}{q(\zvec |\xvec, \d)}\cdot p_{\mathcal{N}(0,I)}(\zvec) \big]\big) \notag \\
    &\geq \min_{q(\zvec|\xvec,\d)} \GJSD(\{q(z|d)\}_{d=1}^k).\,
\end{aligned}
\vspace{0.5em}
\end{equation}
We revisit and compare to this and other VAE-based methods \citep{gupta2021controllable,moyer2018invariant} in detail in \autoref{sec:revisiting-vae}.

\paragraph{Flow-based Methods}
Leveraging the development of invertible normalizing flows \citep{papamakarios2021normalizing}, \citet{grover2020alignflow} proposed a combination of flow-based and adversarial distribution matching objectives for domain adaptation.
\citet{usman2020log} proposed another upper bound for flow-based models  where one distribution is fixed.
Recently, \citet{cho2022cooperative} generalized prior flow-based methods under a common framework by proving that the following flow-based alignment upper bound (AUB) is an \emph{upper bound} on GJSD\footnote{refer to \autoref{appendix:AUB} for more background on AUB}:
\vspace{1em}
\begin{equation}
    \begin{aligned}
    &\min_\g \left(\min_{p(\zvec) \in \pset} \E_{q(\xvec, \zvec, \d)}\big[-\log \big(|J_{g}(\xvec| \d)|\cdot p(\zvec)\big) \big]\right) \notag \\
    &=\min_\g \mname(\g) \geq \min_\g \GJSD(\{q(z|d)\}_{d=1}^k).\,,
    \end{aligned}
    \vspace{0.5em}    
\end{equation}

Like VAE-based methods, this forms a min-min problem but the optimization is over the prior distribution $p(\zvec)$ rather than a decoder.
But, unlike VAE-based methods, AUB does not impose any constraints on the shared latent distribution $p(\zvec)$, i.e., it does not have to be a fixed latent distribution.
While AUB \citep{cho2022cooperative} provides an elegant characterization of distribution matching theoretically, the implementation of AUB still suffers from two issues.
First, AUB assumes that its aligner $g$ is invertible, which requires specialized architectures and can be challenging to optimize in practice.
Second, AUB inherits the vanishing gradient problem of theoretic JSD even if the bound is made tight---which was originally pointed out by \citet{arjovsky2017towards}.
We aim to address these two issues in the subsequent sections and then revisit VAE-based matching to see how our resulting matching loss is both similar and different from prior VAE-based methods.

\section{RELAXING INVERTIBILITY CONSTRAINT OF AUB VIA VAES}
\label{sec:relaxing-invertibility}

One key limitation of the AUB matching measure is that $\g$ must be invertible, which can be challenging to enforce and optimize.
Yet, the invertibility of $\g$ provides two distinct properties. 
First, invertibility enables exact log likelihood computation via the change-of-variables formula for invertible functions. 
Second, invertibility perfectly preserves mutual information between the observed and latent space conditioned on the domain label, i.e., $I(\xvec,\zvec|\d)$ achieves its maximal value.
Therefore, we aim to relax invertibility while seeking to retain the benefits of invertibility as much as possible.
Specifically, we approximate the log likelihood via a VAE approach, which we show is a true relaxation of the Jacobian determinant computation, and we add a mutual information term $I(\xvec,\zvec|\d)$, which attains its maximal value if the encoder is invertible. 
These two relaxation steps together yield an distribution matching objective that is mathematically similar to the domain-conditional version of $\beta$-VAE \citep{higgins2017betavae} where $\beta \leq 1$.
Finally, we propose a plug-and-play version of our objective that can be used as a drop-in replacement for adversarial loss terms so that the matching bounds can be used in any model pipeline.


\subsection{VAE-based Alignment Upper Bound (VAUB)}
We will first relax invertibility by replacing $\g$ with a stochastic autoencoder, where $q(\zvec|\xvec,\d)$ denotes the encoder and $p(\xvec | \zvec, \d)$ denotes the decoder.
As one simple example, the encoder could be a deterministic encoder $\g$ plus some learned Gaussian noise, i.e., $q_{\g}(\zvec | \xvec, \d) = \normal(\g(\xvec|\d), \sigma^2(\xvec,\d) I)$.
The marginal latent encoder distribution is $q(\zvec|\d) = \int_{\xset} q(\xvec|\d) q(\zvec|\xvec,\d)\mathrm{d}\xvec$.
Given this, we can define an VAE-based objective and prove that it is an upper bound on the GJSD.
All proofs are provided in \autoref{appendix:proof}.
\vspace{0.5em}
\begin{definition}[VAE Alignment Upper Bound (VAUB)]
The $\mathrm{VAUB}(q(\zvec|\xvec,\d))$ of a probabilistic aligner (i.e., encoder) $q(\zvec|\xvec,\d)$ is defined as:
\vspace{1em}
\begin{equation}
    \begin{aligned}
    \min_{\substack{ p(\zvec) \\ p(\xvec|\zvec,\d)}} 
        \E_{q(\xvec,\zvec,\d)}\left[-\log \frac{p(\xvec|\zvec, \d)}{q(\zvec |\xvec, \d)}\cdot p(\zvec) \right]
        + C , 
\end{aligned}
\vspace{0.5em}
\end{equation}
where $C \triangleq -\E_{q(\d)}[\HH(q(\xvec|\d))]$ is constant w.r.t.  $q(\zvec|\xvec,\d)$ and thus can be ignored during optimization.
\end{definition}
\vspace{0.5em}
\begin{theorem}[VAUB is an upper bound on GJSD]
\label{thm:vaub-upper-bound}
VAUB is an upper bound on GJSD between the latent distributions $\{q(\zvec|\d)\}_{\d=1}^\ndist$ with a bound gap of $\KL(q(\zvec), p(\zvec)) + \E_{q(\d)q(\zvec|\d)}[\KL(q(\xvec|\zvec,\d),p(\xvec|\zvec,\d))]$ that can be made tight if the $p(\xvec|\zvec,\d)$ and $p(\zvec)$ are optimized over all possible densities. 
\end{theorem}
While prior works proved a similar bound \citep{louizos2015variational,gupta2021controllable}, an important difference is that this bound can be made tight if optimized over the shared latent distribution $p(\zvec)$, whereas prior works assume $p(\zvec)$ is a fixed normal distribution (more details in \autoref{sec:revisiting-vae}).
Thus, VAUB is a more direct relaxation of AUB, which optimizes over $p(\zvec)$.
Another insightful connection to the flow-based AUB \cite{cho2022cooperative} is that the term $\E_{q(\zvec|\xvec,\d)}[-\log \frac{p(\xvec|\zvec, \d)}{q(\zvec |\xvec, \d)}]$ can be seen as an upper bound generalization of the $-\log |J_\g(\xvec|\d)|$, similar to the correspondence noticed in \cite{nielsen2020survae}.
The following proposition proves that this term is indeed a strict generalization of the Jacobian determinant term.
\vspace{0.5em}
\begin{proposition}
\label{thm:log-determinant}
    If the decoder is optimal, i.e., $p(\xvec|\zvec,\d)=q(\xvec|\zvec,\d)$, then the decoder-encoder ratio is the ratio of the marginal distributions: $\E_{q(\zvec|\xvec,\d)}\Big[-\log \frac{p(\xvec|\zvec,\d)}{q(\zvec|\xvec,\d)}\Big]= \E_q\Big[-\log\frac{q(\xvec|\d)}{q(\zvec|\d)}\Big]$.
    If the encoder is also invertible, i.e., $q(\zvec|\xvec,\d) = \delta(\zvec-g(\xvec))$, where $\delta$ is a Dirac delta, then the ratio is equal to the Jacobian determinant: $\E_{q(\zvec|\xvec,\d)}\Big[-\log \frac{p(\xvec|\zvec,\d)}{q(\zvec|\xvec,\d)}\Big]=-\log |J_g(\xvec|\d)|$.
\end{proposition}

This proposition gives a stochastic version of the change of (probability) volume under transformation.
In the invertible case, this is captured by the Jacobian determinant. While for the stochastic case, given an input $\xvec$, consider sampling multiple latent points from the posterior $q(\zvec|\xvec,d)$, which can be thought of as a posterior mean prediction plus some small noise.
Now take the expected ratio between the marginal densities of $q(\zvec|d)$ for each sample point and $q(\xvec|d)$.
If on average $q(\xvec|d)/q(\zvec|d) > 1$, then the transformation locally expands the space (akin to determinant greater than 1) and vice versa.
This ratio estimator can consider volume changes due to non-invertibility and stochasticity. For example, $\zvec = b\xvec + \epsilon$ for some $b<1$ and independent noise $\epsilon$, locally ``shrinks'' because $b<1$ but also locally expands because the noise flattens the distribution.

While the VAUB looks similar to standard VAE objectives, the key difference is noticing the role of $\d$ in the bound.
Specifically, the encoder and decoder can be conditioned on the domain $\d$ but the trainable prior $p(\zvec)$ is \emph{not conditioned on the domain $\d$}.
This shared prior ties all the latent domain distributions together so that the optimal is only achieved when $q(\zvec|\d)=q(\zvec)$ for all $\d$.
Additionally, the perspective here is flipped from the VAE generative model perspective; rather than focusing on the generative model $p$ the goal is finding the encoder $q$ while $p$ is seen as a variational distribution used to learn $q$.
Finally, we note that VAUB could accommodate the case where the encoder is shared, i.e., it does not depend on $\d$ so that $q(\zvec|\xvec,\d) = q(\zvec|\xvec)$.
However, the dependence of the decoder $p(\xvec|\zvec,\d)$ on $\d$ should be preserved (otherwise the domain information would be totally ignored and distribution matching would not be enforced).

\subsection{Preserving Mutual Information via Reconstruction Loss}
While the previous section proved an alignment upper bound for probabilistic aligners based on VAEs, we would also like to preserve the property of flow-based methods that preserves the mutual information between the observed and latent spaces.
Formally, for flow-based aligners $\g$, we have that by construction $I(\xvec,\zvec|\d)=I(\xvec,g^{-1}(\zvec|\d)|\d)=I(\xvec,\xvec|\d)=\HH(\xvec|\d)$, i.e., no information is lost.
Instead of requiring exact invertibility, we relax this property by maximizing the mutual information between $\xvec$ and $\zvec$ given the domain $\d$.
Mutual information can be lower bounded by the negative log likelihood of a decoder (i.e., the reconstruction loss term of VAEs), i.e., $I(\xvec,\zvec|\d) \geq \max_{\tilde{p}(\xvec|\zvec,\d)} \E_{q(\xvec|\d)}[\log \tilde{p}(\xvec|\zvec,\d)] + C$, where $\tilde{p}$ is a variational decoder and $C$ is independent of model parameters (though well-known, we include the proof in \autoref{appendix:proof} for completeness).
Similar to the previous section, this relaxation is a strict generalization of invertibility in the sense that mutual information is maximal in the limit of the encoder being exactly invertible.
While technically this decoder $\tilde{p}$ could be different from the alignment-based $p$, it is natural to make them the same so that $\tilde{p}(\xvec|\zvec,\d)\triangleq p(\xvec|\zvec,\d)$.
Therefore, this additional reconstruction loss can be directly combined with the overall objective:
\vspace{1em}
\begin{equation}
\begin{aligned}
    &\GJSD(\{q(\zvec|\d)\}_{d=1}^k) + \lambda\E_{q}[-I(\xvec,\zvec|\d)] + C \notag\\
    &\leq\!\!\! \min_{\substack{ p(\zvec) \\ p(\xvec|\zvec,\d)}} 
        \underbrace{\E_{q}\!\Big[\!\!-\!\log \frac{p(\xvec|\zvec, \d)}{q(\zvec |\xvec, \d)} p(\zvec) \Big]}_{\text{VAUB objective}}
     +\lambda \underbrace{\E_{q}[-\log p(\xvec|\zvec,\d)]}_{\text{Bound on $I(\zvec,\xvec|\d)$}} \notag \\
    &= \!\!\!\min_{\substack{ p(\zvec) \\ p(\xvec|\zvec,\d)}} \!\!\!
        {\textstyle \frac{1}{\beta}}\underbrace{\E_{q}\!\Big[\!\!-\!\log \frac{p(\xvec|\zvec, \d)}{q(\zvec |\xvec, \d)^\beta} p(\zvec)^\beta \Big]}_{\text{$\beta$-VAUB obj with $\beta \triangleq \frac{1}{1+\lambda}$}}
\end{aligned}
\vspace{0.5em}
\end{equation}
where $\lambda \geq 0$ is the mutual information regularization and $\beta \triangleq \frac{1}{1+\lambda} \leq 1$ is a hyperparameter reparametrization that matches the form of $\beta$-VAE \citep{higgins2017betavae}.
While the form is similar to a vanilla $\beta$-VAE (except for conditioning on the domain label $\d$), the goal of $\beta$ here is to encourage good reconstruction while also ensuring distribution matching rather than making features independent or more disentangled as in the original paper \citep{higgins2017betavae}.
Therefore, we always use $\lambda \geq 0$ (or equivalently $\beta\leq 1$).
As will be discussed when comparing to other VAE-based matching methods, this modification is critical for the good performance of VAUB, particularly to avoid posterior collapse.
Indeed, posterior collapse can satisfy latent distribution matching but would not preserve any information about the input, i.e., if $q(\zvec|\xvec,\d) = q_{\mathcal{N}(0,I)}(\zvec)$, then  the latent distributions will be trivially matched but no information will be preserved, i.e., $I(\xvec,\zvec|\d) = 0$.

\subsection{Plug-and-Play Matching Loss}
While it may seem that VAUB requires a VAE model, we show in this section that VAUB can be encapsulated into a self-contained loss function similar to the self-contained adversarial loss function.
Specifically, using VAUB, we can create a plug and play matching loss that can replace any adversarial loss with a non-adversarial counterpart \emph{without requiring any architecture changes to the original model}.
\begin{definition}[Plug-and-play matching loss] \label{eqn:vaub_pnp}
   Given a deterministic feature extractor $\g$, let  $q_{\g,\sigma^2}(\zvec|\xvec,\d) \triangleq \mathcal{N}(\g(\xvec|\d), \textnormal{diag}(\sigma^2(\xvec,\d)))$ be a simple probabilistic version of $\g$ where $\sigma^2(\xvec,\d)$ is a trainable diagonal convariance matrix.
   Then, the plug-and-play matching loss for any deterministic representation function $\g$ can be defined as:
   \vspace{0.5em}
   \begin{align}
       &\textnormal{VAUB\_PnP}(\g) \triangleq \beta\textnormal{-VAUB}(q_{\g,\sigma^2}(\zvec|\xvec,\d)) \\
       &=\!\!\!\!\min_{\substack{\sigma^2(\xvec,\d) \\ p(\xvec|\zvec,\d), p(\zvec)}}\!\!\! \E_{q}\left[-\log \left(\frac{p(\xvec|\zvec, \d)}{q_{\g,\sigma^2}(\zvec |\xvec, \d)^\beta} p(\zvec)^\beta\right)\right] . \notag
   \end{align}
\end{definition}
Like the adversarial loss in \autoref{eqn:adversarial-loss}, this loss is a self-contained variational optimization problem where the auxiliary models (i.e., discriminator for adversarial and decoder distributions for VAUB) are only used for distribution matching optimization.
This loss does not require the main pipeline to be stochastic and these auxiliary models could be simple functions.
While weak auxiliary models for adversarial could lead to an arbitrarily poor approximation to GJSD, our upper bound is guaranteed to be an upper bound even if the auxiliary models are weak.
Thus, we suggest that our $\beta$-VAUB\_PnP loss function can be used to safely replace any adversarial loss function.

\section{NOISY JENSEN-SHANNON DIVERGENCE}
\label{sec:noisy-jsd}
While in the previous section we addressed the invertibility limitation of AUB, we now consider a different issue related to optimizing a bound on the JSD.
As has been noted previously \cite{arjovsky2017wasserstein}, the standard JSD can saturate when the distributions are far from each other which will produce vanishing gradients.   
While \cite{arjovsky2017wasserstein} switch to using Wasserstein distance instead of JSD, we revisit the idea of adding noise to the JSD as in the predecessor work \cite{arjovsky2017towards}.
\cite{arjovsky2017towards} suggest smoothing the input distributions with Gaussian noise to make the distributions absolutely continuous and prove that this noisy JSD is an upper bound on the Wasserstein distance.
However, in \cite{arjovsky2017towards}, the JSD is estimated via an adversarial loss, which is a \emph{lower bound} on JSD. Thus, it is incompatible with their theoretic upper bound on Wasserstein distance.
In contrast, because we have proven an upper bound on JSD, we can also consider a upper bound on noisy JSD that could avoid some of the problems with the standard JSD.
We first define noisy JSD and prove that it is in fact a true divergence.
\begin{definition}[Noisy JSD]
Noisy JSD is the JSD after adding Gaussian noise to the distributions, i.e.,
$
    \mathrm{NJSD}_{\sigma}(p,q) = \mathrm{JSD}(\tilde{p}_\sigma, \tilde{q}_\sigma) ,
$
where $\tilde{p}_{\sigma} \triangleq p \ast \mathcal{N}(0, \sigma^2 I )$ ($\ast$ denotes convolution) and similarly for $\tilde{q}_{\sigma}$.
\end{definition}
Note that in terms of random variables, if $x \sim p_x$, $y \sim p_y$, and $\epsilon \sim \normal(0, \sigma^2 I)$, then $\mathrm{NJSD}(p_x, p_y) = \mathrm{NJSD}(p_{x+\epsilon}, p_{y+\epsilon})$.
We prove that Noisy JSD is indeed a statistical divergence using the properties of JSD and the fact that convolution with a Gaussian density is invertible.
\begin{proposition}
\label{thm:nsj-is-divergence}
    Noisy JSD is a statistical divergence.
\end{proposition}


%
%
%

In \autoref{appendix-NJSD-demo}, we present a toy example illustrating how adding noise to JSD can alleviate plateaus in the theoretic JSD that can cause vanishing gradient issues and can smooth over local minimum in the optimization landscape.
We now prove noisy versions of both AUB and VAUB to be upper bounds of Noisy JSD.
To the best of our knowledge, these bounds are novel though straightforward in hindsight.
\begin{theorem}[Noisy alignment upper bounds]
\label{thm:noisy-upper-bounds}
For the flow-based AUB, the following upper bound holds for :
\vspace{0.5em}
\begin{equation}
\begin{aligned}
    &\textnormal{NAUB}(q(\zvec|\xvec,\d); \sigma^2)  + C \\
    &\triangleq \min_{p(\tilde{\zvec})} 
    \E_{q(\xvec,d)q(\epsilon;\sigma^2)}[-\log |J_g(\xvec|d)| p(g(\xvec|d) + \epsilon)] \notag \\
    &\geq \textnormal{NJSD}(\{q(\zvec|\d)\}_{\d=1}^\ndist; \sigma^2) \,,\notag
\end{aligned}
\vspace{0.5em}
\end{equation}

where $C\triangleq \E_{q(\d)}[ \HH(q(\xvec|\d))]$.
Similarly, for VAUB, the following upper bound holds:
\vspace{0.5em}
\begin{equation}
\begin{aligned}
    &\textnormal{NVAUB}(q(\zvec|\xvec,\d); \sigma^2) +C \\
    &\triangleq \min_{\substack{p(\tilde{\zvec}) \\ 
    p(\xvec|\zvec,\d)}} 
    \E_{q(\xvec,\zvec,\d)q(\epsilon;\sigma^2)}\left[-\log \left(\frac{p(\xvec|\zvec, \d)}{q(\zvec |\xvec, \d)}\cdot p(\zvec + \epsilon)\right)\right] \notag \\
    &\geq \textnormal{NJSD}(\{q(\zvec|\d)\}_{\d=1}^\ndist; \sigma^2) \,.\notag
\end{aligned}
\end{equation}
\end{theorem}
By comparing the original objectives and these noisy objectives, we notice the correspondence between adding noise before passing to the shared distribution $p(\zvec + \epsilon)$ and the noisy JSD.
This suggests that simple additive noise can add an implicit regularization that could make the optimization smoother.
Similar to VAUB, a $\beta$-VAUB version of these can be used to preserve mutual information between $\xvec$ and $\zvec$.


    




\section{REVISITING VAE-BASED MATCHING METHODS FROM FAIRNESS LITERATURE}
\label{sec:revisiting-vae}
The literature on fair classification has proposed several VAE-based methods for distribution matching.
For fairness applications, the global objective includes both a classification loss and an matching loss but we will only analyze the matching losses in this paper.
\citet{louizos2015variational} first proposed the vanilla form of a VAE with an matched latent space where the prior distribution is fixed.
\cite{moyer2018invariant} and \cite{gupta2021controllable} take a mutual information perspective and bound two different mutual information terms in different ways.
They formulate the problem as minimizing the mutual information between $\zvec$ and $d$, where $d$ corresponds to their sensitive attribute---our generalized JSD is in fact equivalent to this mutual information term, i.e., $\GJSD(\{q(\zvec|d)\}_{d=1}^k) \equiv I(\zvec,d)$.
They then use the fact that $I(\zvec,d) = I(\zvec,d|\xvec) + I(\zvec,\xvec) - I(\zvec,\xvec|d) = I(\zvec,\xvec) - I(\zvec,\xvec|d)$, where the second equals is by the fact that $\zvec$ is a deterministic function of $\xvec$ and independent noise.
Finally, they bound $I(\zvec,\xvec)$ and $-I(\zvec,\xvec|d)$ separately.
We explain important differences here and point the reader to the \autoref{appendix:comparison-table} for a detailed comparison between methods.

We notice that most prior VAE-based methods use a fixed standard normal prior distribution $p_\mathcal{N}(\zvec)$.
This can be seen as a special case of our method in which the prior is not learnable.
However, a fixed prior actually imposes constraints on the latent space beyond distribution matching, which we formalize in this proposition.
\begin{proposition}
\label{thm:fixed-prior}
The fair VAE objective from \citet{louizos2015variational} with a fixed latent distribution $p_{\mathcal{N}(0,I)}(\zvec)$ can be decomposed into a VAUB term and a regularization term on the latent space:
\vspace{0.5em}
\begin{align*}
    &\min_{q(\zvec|\xvec,d)} \min_{p(\xvec|\zvec,d)} \E_{q}\big[-\log \frac{p(\xvec|\zvec,d)}{q(\zvec|\xvec,d)} p_{\mathcal{N}(0,I)}(\zvec) \big] \\
    &=\min_{q(\zvec|\xvec,d)} \underbrace{\Big(\min_{p(\xvec|\zvec,d)} \min_{p(\zvec)} \E_{q}\big[-\log \frac{p(\xvec|\zvec,d)}{q(\zvec|\xvec,d)} p(\zvec) \big] \Big)}_{\text{VAUB Alignment Objective}} \\
    &\quad\quad\quad\quad\quad\quad + \underbrace{\KL(q(\zvec),p_{\mathcal{N}(0,I)}(\zvec))}_{\text{Regularization}} \,.
\end{align*}
\end{proposition}
This proposition highlights that prior VAE-based matching objectives are actually solving distribution matching \emph{plus a regularization term} that pushes the learned latent distribution to the normal distribution---i.e., they are biased distribution matching methods.
In contrast, GAN-based matching objectives do not have this bias as they do not assume anything about the latent space.
Similarly, our VAUB methods can be seen as relaxing this by optimizing over a class of latent distributions for $p(\zvec)$ to reduce the bias.

Furthermore, we notice that \cite{moyer2018invariant} used a similar term as ours for $-I(\zvec,\xvec|d)$ but used a non-parametric pairwise KL divergence term for $I(\zvec,\xvec)$, which scales quadratically in the batch size.
On the other hand, \cite{gupta2021controllable} uses a similar variational KL term as ours for $I(\zvec,\xvec)$ but decided on a contrastive mutual information bound for $-I(\zvec,\xvec|d)$.
\cite{gupta2021controllable} did consider using a similar term for $-I(\zvec,\xvec|d)$ but ultimately rejected this alternative in favor of the contrastive approach. 
We summarize the differences as follows:
\begin{enumerate}[wide=0pt,leftmargin=*,topsep=0pt,itemsep=-1ex,partopsep=1ex,parsep=1ex]
    \item We allow the shared prior distribution $p_\theta(\zvec)$ to be \emph{learnable} so that we do not impose any distribution on the latent space. Additionally, optimizing $p_\theta(\zvec)$ is significant as we show in \autoref{thm:vaub-upper-bound} that this can make our bound tight.
    \item The $\beta$-VAE change ensures better preservation of the mutula information of $\xvec$ and $\zvec$ inspired by the invertible models of AUB. This seemingly small change seems to overcome the limitation of the reconstruction-based approach originally rejected in \cite{gupta2021controllable}.
    \item We propose a novel noisy version of the bound that can smooth the optimization landscape while still being a proper divergence.
    \item We propose a \emph{plug-and-play} version of our bound that can be added to \emph{any} model pipeline and replace \emph{any} adversarial loss.
    Though not a large technical contribution, this perspective decouples the distribution matching loss from VAE models to create a self-contained distribution matching loss that can be broadly applied outside of VAE-based models.
\end{enumerate}

\section{EXPERIMENTS}


\subsection{Simulated Experiments} \label{sec:toy_experiment}

\paragraph{\changeziyu{Non-Matching Dimensions between Latent Space and Input Space}}
In order to demonstrate that our model relaxes the invertibility constraint of AUB, we use a dataset consisting of rotated moons where the latent dimension does not match the input dimension (i.e. the transformation between the input space and latent space is not invertible). 
Please note that the AUB model proposed by \cite{cho2022cooperative} is not applicable for such situations due to the requirement that the encoder and decoder need to be invertible, which restricts our ability to select the dimensionality of the latent space.
As depicted in \autoref{fig:rotated_moon}, our model is able to effectively reconstruct and flip the original two sample distributions despite lacking the invertible features between the encoders and decoders. Furthermore, we observe the mapped latent two sample distributions are matched with each other while sharing similar distributions to the shared distribution $p(z)$.

\paragraph{Noisy-AUB helps mitigate the Vanishing Gradient Problem}
In this example, we demonstrate that the optimization can get stuck in a plateau region without noise injection. However, this issue can be resolved using the noisy-VAUB approach. Initially, we attempt to match two Gaussian distributions with widely separated means. As shown in \autoref{fig:nvaub}, the shared distribution, $p(\zvec)$ (which is a Gaussian mixture model), initially fits the bi-modal distribution, but this creates a plateau in the optimization landscape with small gradient even though the latent space is misaligned. While VAUB can eventually escape such plateaus with sufficient training time, these plateaus can unnecessarily prolong the training process. In contrast, NVAUB overcomes this issue by introducing noise in the latent domain and thereby reducing the small gradient issue.

\begin{figure}[t]
\centering

\includegraphics[width=\linewidth]{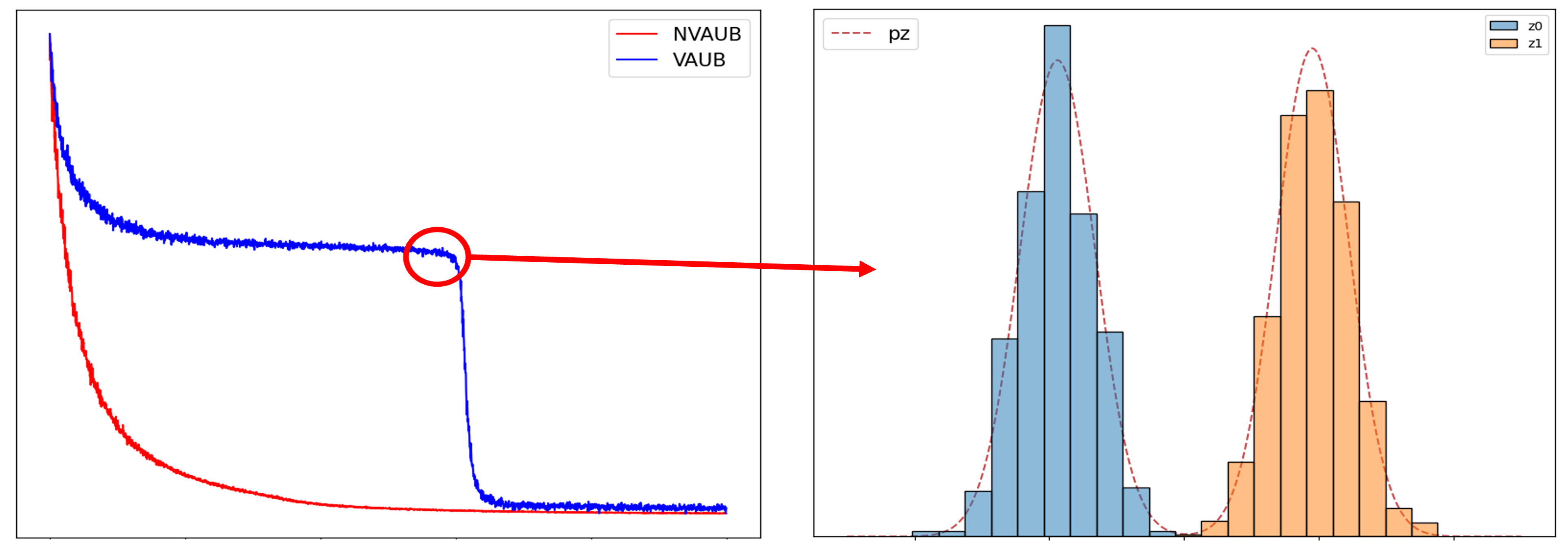}

\caption{This figure shows that the loss reaches a plateau during learning if VAUB is used. (a) shows the loss convergence graph for VAUB and NVAUB, while (b) visualizes the latent distribution $p(\zvec)$ with histogram density estimation of $z_i \sim q(z|x, d=i), i\in{0,1}$ at the \textbf{red circle} in figure (a). Notice that the latent distribution matches the mixture of the domains but the latent domain distributions are not yet aligned.}
\label{fig:nvaub}
\end{figure}

\begin{figure*}[ht]
\minipage{0.24\textwidth}
  \includegraphics[width=\linewidth]{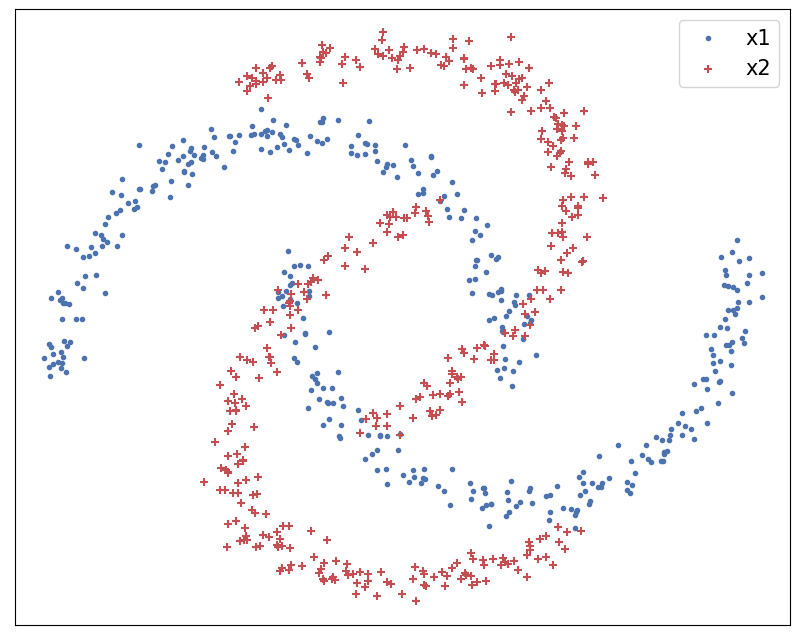}
  \caption*{(a) Original dist.}
\endminipage\hfill
\minipage{0.24\textwidth}
  \includegraphics[width=\linewidth]{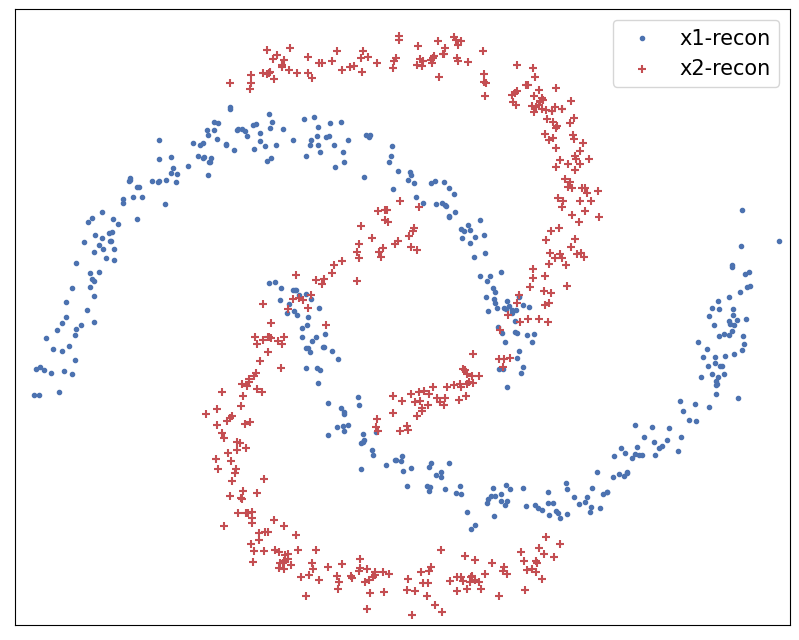}
  \caption*{(b) Reconstructed dist.}
\endminipage\hfill
\minipage{0.24\textwidth}
  \includegraphics[width=\linewidth]{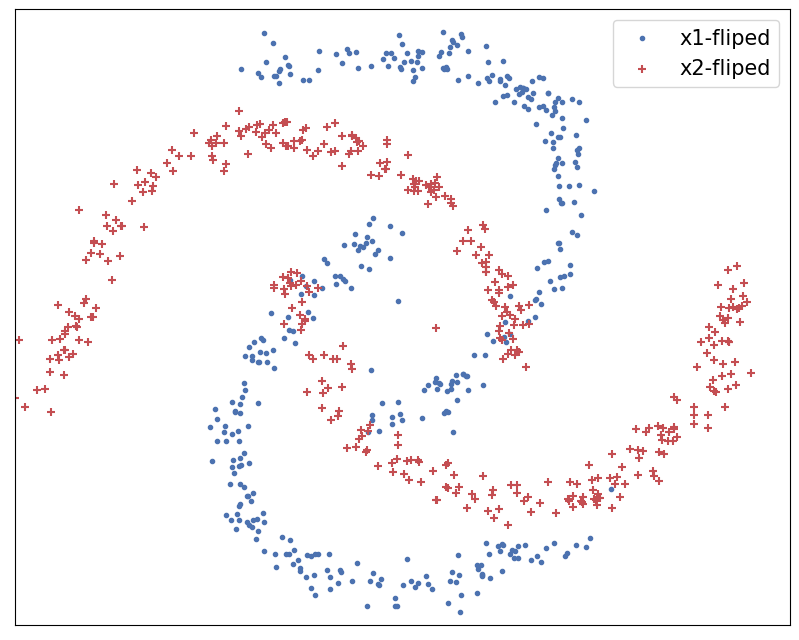}
  \caption*{(c) Transformed dist.}\label{fig:awesome_image2}
\endminipage\hfill
\minipage{0.24\textwidth}
  \includegraphics[width=\linewidth]{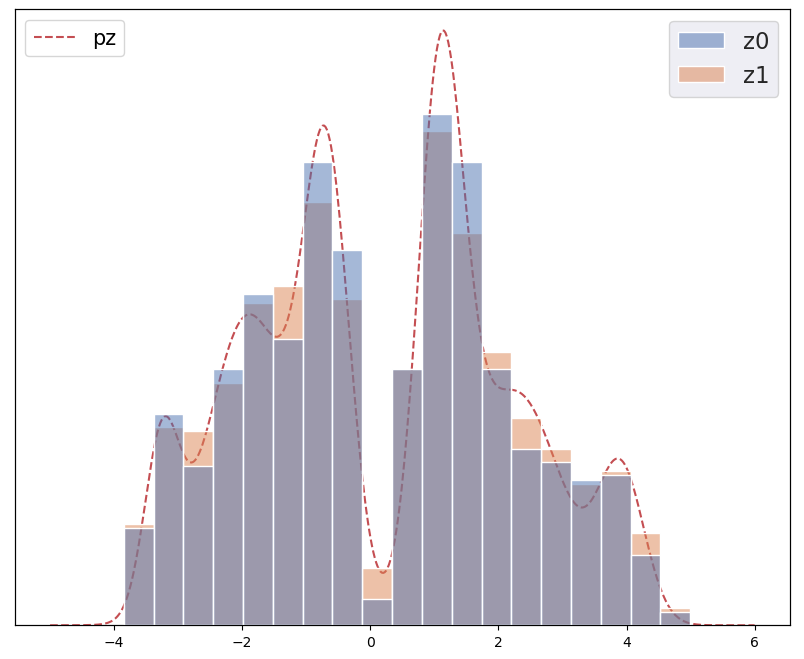}
  \caption*{(d) Latent dist. with $p(\zvec)$}\label{fig:awesome_image3}
\endminipage
\caption{
This figure shows distribution matching on the 2D rotated moons dataset while the latent sample distribution is 1D. (a) shows the original distribution where $x_1$ is the original moons dataset and $x_2$ is the rotated and scaled moons dataset. (b) shows the reconstructed distribution generated by using the probabilistic encoder and decoder within the same domain, i.e. $x^{recon}_i\sim p(x|z,d_{=i})$ (c) shows the flipped distribution between two distributions by forwarding the probabilistic encoder and decoder in different domains, i.e. $x^{flipped}_i\sim p(x|z,d_{=j})$, where the dataset is flipped from $i$ to $j$. (d) shows the shared distribution $p(\zvec)$, along with the histogram density estimation of the latent distributions of $z_i \sim q(z|x, d=i), i\in{0,1}$.}
\label{fig:rotated_moon}
\end{figure*}


\subsection{Other Non-adversarial Bounds}
\label{sec:experiments-comparison-other-bounds}

Due to the predominant focus on fairness learning in related works, we select the \emph{Adult dataset}\footnote{https://archive.ics.uci.edu/ml/datasets/adult}, comprising 48,000 instances of census data from 1994. To investigate distribution matching, we create domain samples by grouping the data based on gender attributes (male, female) and aim to achieve matching between these domains. We employ two alignemnt metrics: sliced Wasserstein distance (SWD) and a classification-based metric. For SWD, we obtain the latent sample distribution ($z_i \sim q(z|x_i)$) and whiten it to obtain $z^{white}_i$.
Because Wasserstein distance is sensitive to scaling, the whitening step is required to remove the effects of scaling, which do not fundamentally change the distribution matching performance. We then measure the SWD between the whitened sample distributions $z^{white}{male}$ and $z^{white}_{female}$ by projecting them onto randomly chosen directions and calculating the 1D Wasserstein distance.
We use $t$-test to assess significant differences between models.
For the classification-based metric, we train a Support Vector Machine (SVM) with a Gaussian kernel to classify the latent distribution, using gender as the label. A more effective distribution matching model is expected to exhibit lower classification accuracy, reflecting the increasing difficulty in differentiating between matched distributions. 
We chose kernel SVM over training a deep model because the optimization is convex with a unique solution given the hyperparameters.
In all cases, we used cross-validation to select the kernel SVM hyperparameters.

We choose two non-adversarial bounds \cite{moyer2018invariant, gupta2021controllable} as our baseline methods.
Here we only focus on the distribution matching loss functions of all baseline methods.
Note that these methods are closely related (see \autoref{sec:revisiting-vae}).
To ensure a fair comparison, we employ the same encoders for all models and train them until convergence. 
The results in \autoref{tab:compare_bounds} suggest that our VAUB method can achieve better distribution matching compared to prior variational upper bounds.



\begin{table}[htb]
\centering
    \caption{Distribution matching comparison between non-adversarial bounds (refer to \autoref{appendix:experiment-setup} for $t$-test details)}
    \begin{tabular}{cccc}
    \toprule
         & Moyer's  & Gupta's & VAUB  \\ 
    \midrule
    SWD ($\downarrow$)  & 9.71   & 6.70    & $\bm{5.64}$    \\
    SVM ($\downarrow$)  & 0.997  & 0.833   & $\bm{0.818}$   \\
    \bottomrule
    \end{tabular}\label{tab:compare_bounds}
\end{table}

\subsection{Plug-and-Play Implementation}
In this section, we see if our plug-and-play loss can be used to replace prior adversarial loss functions in generic model pipelines rather than those that are specifically VAE-based. 

\paragraph{Replacing the Adversarial Objective in Fairness Representation Models}
In this experiment, we explore the effectiveness of using our min-min VAUB plug-and-play loss as a replacement for the adversarial objective of LAFTR \citep{madras2018learning}. 
We include the LAFTR classifying loss and adjust the trade-off between the classification loss and matching loss via $\lambda_{aub}$ \autoref{eqn:vaub_pnp}.
In this section, we conduct a comparison between our proposed model and the LAFTR model on the Adult dataset where the goal is to fairly predict whether a person's income is above 50K, using gender as the sensitive attribute. 
To demonstrate the plug-and-play feature of our VAUB model, we employ the same network architecture for the deterministic encoder $g$ and classifier $h$ as in LAFTR-DP. We also share the parameters of our encoders (i.e. $q(z|x,d)=q(z|x)$) to comply with the structure of the LAFTR.
The results of the fair classification task are presented in \autoref{tab:fair-table}, where we evaluate using three metrics: overall accuracy, demographic parity gap ($\Delta_{DP}$), and test VAUB loss. We argue that our model(VAUB) still maintains the trade-off property between classification accuracy and fairness while producing similar results to LAFTR-DP. 
Moreover, we observe that LAFTR-DP cannot achieve perfect fairness in terms of demographic parity gap by simply increasing the fairness coefficient $\gamma$, whereas our model can achieve a gap of 0 with only a small compromise in accuracy. 
Finally, we note that the VAUB loss correlates well with the demographic parity gap, indicating that better distribution matching leads to improved fairness; perhaps more importantly, this gives some evidence that VAUB may be useful for measuring the relative distribution matching between two approaches.
Please refer to the appendix for a more comprehensive explanation of all the experiment details.

\begin{table}[ht]
\centering
\caption{Accuracy, Demographic Parity Gap ($\Delta_{DP}$), and VAUB Metric (in nats) on Adult Dataset with the models VAUB and LAFTR-DP. Numbers after VAUB indicate $\lambda_{aub}$, and numbers after LAFTR-DP indicate the fairness coefficient $\gamma$.}
\label{tab:fair-table}
\begin{tabular}{cccc}
\toprule
\textbf{Method} & \textbf{Accuracy} & \textbf{$\Delta_{DP}$} & \textbf{VAUB} \\
\midrule
VAUB(100)          & 0.74              & 0                     & 308.36         \\
VAUB(50)         & 0.76              & 0.0002                & 318.31         \\
VAUB(10)         & 0.797             & 0.058                 & 324.28         \\
VAUB(1)        & 0.835             & 0.252                 & 333.47         \\
\midrule
LAFTR-DP(1000)    &  0.765             & 0.015                 & -          \\
LAFTR-DP(4)      & 0.77              & 0.022                 & -              \\
LAFTR-DP(0.1)   & 0.838             & 0.21                  & -  \\
\bottomrule
\end{tabular}
\end{table}




\paragraph{Comparison of Training Stability between Adversarial Methods and Ours}
To demonstrate the stability of our non-adversarial training objective, we use plug-and-play VAUB model to replace the adversarial objective of DANN \citep{ganin2016domain}. 
We re-define the min-max objective of DANN to a min-min optimization objective while keeping the model the same. 
Because of the plug-and-play feature, we can use the exact same network architecture of the encoder (feature extractor) and label classifier as proposed in DANN and only replace the adversarial loss.
We conduct the experiment on the MNIST \citep{lecun2010mnist} and MNIST-M \citep{ganin2016domain} datasets, using the former as the source domain and the latter as the target domain.
We present the results of our model in \autoref{tab:table:da-fixed-pz}\footnote{We optimized the listed results for the DANN experiment to the best of our ability.}, which reveals a comparable or better accuracy compared to DANN. 
Notably, the NVAUB approach exhibits further improvement over accuracy, as adding noise may smooth the optimization landscape.
In contrast to the adversarial method, our model provides a reliable metric (VAUB loss) for assessing adaptation performance. The VAUB loss shows a strong correlation with test accuracy, whereas the adversarial method lacks a valid metric for evaluating adaptation quality (see \autoref{appendix:adv-vaub-compare} for figure).
This result on DANN and the previous on LAFTR give evidence that our method could be used as drop-in replacements for adversarial loss functions while retaining the performance and matching of adversarial losses.

\begin{table}[!ht]
\centering
    \caption{The table shows the accuracy after the domain adaptation in DANN, VAUB and NVAUB models. For each model, results are averaged from 5 experiments with different random seeds.}
    \begin{tabular}{cccc}
    \toprule
    Method & DANN & VAUB & NVAUB \\ 
    \midrule
    Accuracy ($\uparrow$) & 75.42 & 75.53 & $\mathbf{76.47}$ \\ 
    \bottomrule
    \end{tabular}\label{tab:table:da-fixed-pz}
\end{table}


\section{LIMITATION}

\paragraph{Empirical Scope} Since this paper is theoretically focused rather than empirically focused, our goal was to prove theoretic bounds for distribution matching, elucidate insightful connections to prior works (AUB, fair VAE, and GAN methods), and then empirically validate our methods compared to other related approaches on simple targeted experiments, which means
the experiments conducted in our study are deliberately simple and may not fully capture the complexity of real-world scenarios. We primarily utilize toy datasets such as the 2D moons dataset and 1D noisy VAUB illustrations to illustrate key insights. Although we validate our methods on common benchmark datasets like Adult and MNIST, our choice to avoid state-of-the-art methods and complex datasets may limit the generalizability of our findings.

\paragraph{Comparisons with SOTA models} We intentionally select representative methods such as LAFTR and DANN, along with well-known benchmark datasets, to demonstrate the feasibility of our approach. However, this choice may not fully capture the diversity of existing methods and datasets used in practical applications.
Also, our study does not aim for state-of-the-art performance on specific tasks, which may lead to overlooking certain performance metrics or nuances that are crucial in practical applications. While our approach shows promise as an alternative to adversarial losses, further exploration is needed to understand its performance across various metrics and tasks.

\section{DISCUSSION AND CONCLUSION}

In conclusion, we present a model-agnostic VAE-based distribution matching method that can be seen as a relaxation of flow-based matching or as a new variant of VAE-based methods. Unlike adversarial methods, our method is non-adversarial, forming a min-min cooperative problem that provides upper bounds on JSD divergences. We propose noisy JSD variants to avoid vanishing gradients and local minima and develop corresponding alignment upper bounds.  
We compare to other VAE-based bounds both conceptually and empirically showing how our bound differs.
Finally, we demonstrate that our non-adversarial VAUB alignment losses can replace adversarial losses without modifying the original model's architecture, making them suitable for standard invariant representation pipelines such as DANN or fair representation learning.
We hope this will enable distribution matching losses to be applied generically to different problems without the challenges of adversarial losses.

%% file: appendix-aistats.tex

\renewcommand \thepart{}
\renewcommand \partname{}

\doparttoc 
\faketableofcontents 

\addcontentsline{toc}{section}{Appendix} 
\part{Appendix for Towards Practical Non-Adversarial Distribution Alignment} 
\parttoc 

\section{More Background and Theory from AUB \citep{cho2022cooperative}} \label{appendix:AUB}
Alignment Upper Bound (AUB), introduced from \cite{cho2022cooperative}, jointly learns invertible deterministic aligners $\g$ with a shared latent distribution $p(\zvec)$, where $\zvec = g(\xvec|\d)$, or equivalently where $q(\zvec|\xvec,\d)$ is a Dirac delta function centered at $g(\xvec|\d)$. 
In this section, we remember several important theorems and lemmas from \cite{cho2022cooperative} based on invertible aligners $\g$.
These provide both background and formal definitions. Our proofs follow a similar structure to those in \citet{cho2022cooperative}.

\begin{theorem}[GJSD Upper Bound from \cite{cho2022cooperative}]
\label{thm:gjsd-upper-bound}
Given a density model class $\pset$, we form a GJSD variational upper bound: 
\begin{align*}
    &\GJSD(\{q(\zvec|\d)\}_{\d=1}^\ndist) \leq \min_{p(\zvec) \in \pset} \Hc(q(\zvec), p(\zvec)) -\E_{q(\d)}[\HH(q(\zvec|\d))] \,,
    \notag
\end{align*}
where $q(\zvec) = \sumd \int q(\xvec,\d) q(\zvec|\xvec,\d) \mathrm{d}\xvec = \sumd q(\d) q(\zvec|\d)$ is the marginal of the encoder distribution and the bound gap is exactly $\min_{p(\zvec)\in \pset}\KL(q(\zvec), p(\zvec))$.
\end{theorem}

\begin{lemma}[Entropy Change of Variables from \cite{cho2022cooperative}]
\label{thm:entropy-change-of-variables}
Let $\xvec \sim q(\xvec)$ and $\zvec \triangleq g(\xvec) \sim q(\zvec)$, where $g$ is an invertible transformation. The entropy of $\zvec$ can be decomposed as follows:
\begin{align}
    \HH(q(\zvec)) = \HH(q(\xvec)) + \E_{q(\xvec)}[\log |J_g(\xvec)|] \,,
\end{align}
where $|J_g(\xvec)|$ is the absolute value of the determinant of the Jacobian of $g$.
\end{lemma}

\begin{definition}[Alignment Upper Bound Loss from \cite{cho2022cooperative}]
\label{def:alignment-upper-bound-loss}
The alignment upper bound loss for aligner $\g(\xvec|\d)$ that is invertible conditioned on $\d$ is defined as follows:
\begin{align}
    &\mname(\g) \triangleq \min_{p(\zvec) \in \pset} \E_{q(\xvec, \zvec, \d)}\big[-\log \big(|J_{g}(\xvec| \d)|\cdot p(\zvec)\big) \big] \,,
\end{align}
where $\pset$ is a class of shared prior distributions and $|J_{g}(\xvec|\d)|$ is the absolute value of the Jacobian determinant.
\end{definition}

\begin{theorem}[Alignment at Global Minimum of $\mname$ from \cite{cho2022cooperative}]
\label{thm:global-minimum}
If $\mname$ is minimized over the class of all invertible functions, a global minimum of $\mname$ implies that the latent distributions are aligned, i.e., for all $d,d'$, $q(\zvec| d) = q(\zvec| d') \in \pset$. Notably, this result holds regardless of $\pset$. 
\end{theorem}

\section{Illustration of Noisy JSD for Alleviating Vanishing Gradient and Local Minimum} \label{appendix-NJSD-demo}
We present a toy example illustrating how adding noise to JSD can alleviate plateaus in the theoretic JSD that can cause vanishing gradient issues (\autoref{fig:njsd}(a-b)) and can smooth over local minimum in the optimization landscape (\autoref{fig:njsd}(c-d)).

\begin{figure*}[!ht]
\centering
\begin{subfigure}[t]{0.45\textwidth}
  \includegraphics[width=\linewidth]{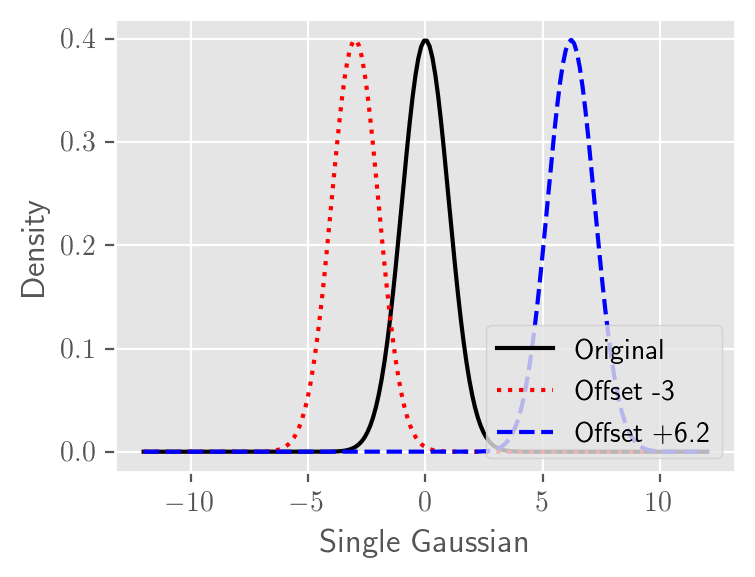}
  \caption{Case 1: Gaussian distributions.}
  \label{subfig:njsd-gaussian}
\end{subfigure}
\begin{subfigure}[t]{0.45\textwidth}
  \includegraphics[width=\linewidth]{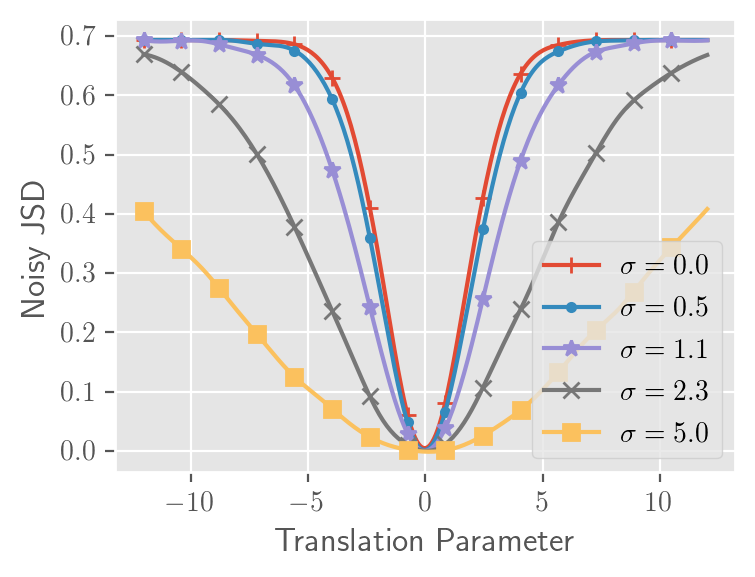}
  \caption{Case 1: Opt. landscape for different noise levels.}
  \label{subfig:njsd-gaussian-optimization}
\end{subfigure}
\begin{subfigure}[t]{0.45\textwidth}
  \includegraphics[width=\linewidth]{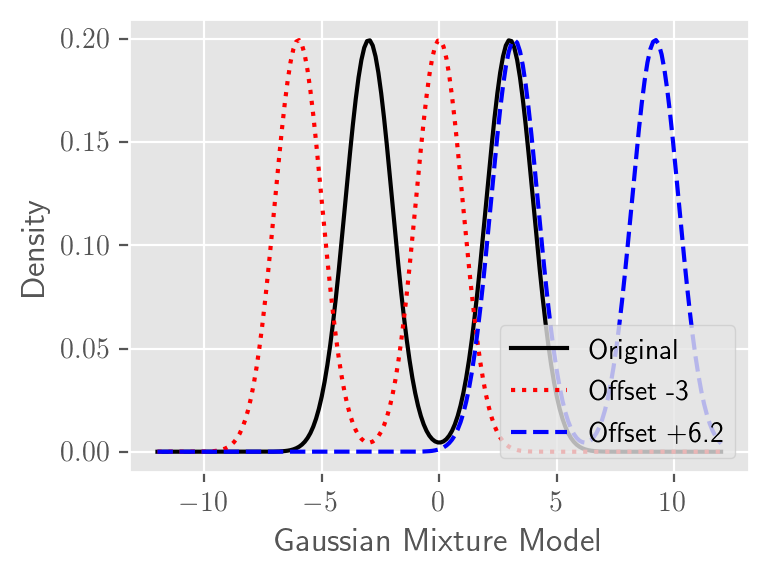}
  \caption{Case 2: Mixture of Gaussian distributions. }
  \label{subfig:njsd-gmm}
\end{subfigure}
\begin{subfigure}[t]{0.45\textwidth}
  \includegraphics[width=\linewidth]{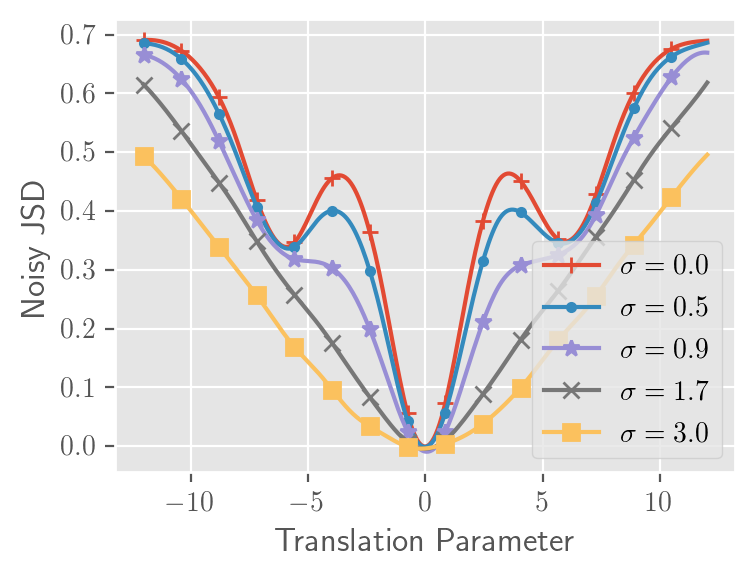}
  \caption{Case 2: Opt. landscape for different noise levels}
  \label{subfig:njsd-gmm-optimization}
\end{subfigure}
\caption{
This figure illustrates that Noisy JSD can reduce the vanishing gradient problem and smooth over local minimum compared to theoretic JSD.
In Case 1 (a), we consider the (Noisy) JSD between two Gaussian distributions whose variance are the same but whose means are different.
The gradient of JSD can vanish to zero as it reaches its maximum value as seen by the plateau regions on the top curve in (b) but this can be alleviated with noise as seen in bottom curves in (b).
In Case 2 (c), we consider the (Noisy) JSD between a Gaussian mixture model where the mixture components are the same but the overall means are different.
For this case, a local minimum of JSD occurs when only one mixture component overlaps as seen in the top curve of (d).
However, Noisy JSD can smooth out this local minimum so that there are no local minimum.
}
\label{fig:njsd}
\end{figure*}

\section{VAE-Based Distribution Alignment Comparison Table} \label{appendix:comparison-table}
We present the full comparison table between VAE-based methods below in \autoref{tab:upper-bound-comparison}. We add the notation where $\phi$ are the parameters of the encoder distributions $q_\phi$ and $\theta$ are the parameters of the decoder distribution $p_\theta$ to emphasize that our VAUB objectives allow training of the prior distribution $p_\theta(\zvec)$, while prior methods assume it is a standard normal distribution.

\begin{table*}[!ht]
    \centering
    \caption{Variational Upper Bounds Comparison: Prior bounds have one or more similarities to our bounds but have several key differences as noted below and explained in this section. $C$ is a constant and CE stands for Contrastive Estimation which equals to $-\log \frac{\exp(f(\zvec,\xvec,d)}{\mathbb{E}_{\zvec'\sim q(\zvec|c)}[\exp(f(\zvec',\xvec,d))]}$.}
    \label{tab:upper-bound-comparison}
    \vspace{-0.5em}
    \begin{tabular}{lll}
        \toprule
        Method & $I(\zvec,\xvec)\leq \cdots$ & $-I(\zvec,\xvec|d) \leq \cdots + C$ \\
        \midrule
        \cite{louizos2015variational} & $\mathbb{E}_{q}[\KL(q_{\phi}(\zvec|\xvec,d), p_\mathcal{N}(\zvec))]$ & $\mathbb{E}_q[-\log p_{\theta}(\xvec|\zvec,d)]$ \\
        \cite{moyer2018invariant} & $\mathbb{E}_{(\xvec,\xvec')\sim q}[\KL(q_{\phi}(\zvec|\xvec), q_{\phi}(\zvec|\xvec'))]$ & $\mathbb{E}_q[-\log p_{\theta}(\xvec|\zvec,d)]$ \\
        \cite{gupta2021controllable} &  $\mathbb{E}_{q}[\KL(q_{\phi}(\zvec|\xvec), p_\mathcal{N}(\zvec))]$ & $\mathbb{E}_{q}[-\textnormal{CE}(\zvec,\xvec,d)]$  \\
        \cite{gupta2021controllable} recon\footnote{The reconstruction-based method was explicitly rejected and not used in \cite{gupta2021controllable} except in an ablation experiment and is thus not representative of the main method proposed by this paper.} & $\mathbb{E}_{q}[\KL(q_{\phi}(\zvec|\xvec), p_\mathcal{N}(\zvec))]$ & $\mathbb{E}_q[-\log p_{\theta}(\xvec|\zvec,d)]$ \\
        \hline
        (ours) VAUB & $\mathbb{E}_{q}[\KL(q_{\phi}(\zvec|\xvec,d), p_\theta(\zvec))]$ & $\mathbb{E}_q[-\log p_{\theta}(\xvec|\zvec,d)]$  \\
        (ours) $\beta$-VAUB ($\beta \leq 1$) & $\mathbb{E}_{q}[\beta \KL(q_{\phi}(\zvec|\xvec,d), p_\theta(\zvec))]$ & $\mathbb{E}_q[-\log p_{\theta}(\xvec|\zvec,d)]$  \\
        (ours) Noisy $\beta$-VAUB & $\mathbb{E}_{q}[
        \underbrace{\beta}_{\text{(2)}} \KL(q_{\phi}(\zvec|\xvec,d), \underbrace{p_\theta}_{\text{(1)}}(\zvec \underbrace{+ \epsilon}_{\text{(3)}}))]$ & $\mathbb{E}_q[-\log p_{\theta}(\xvec|\zvec,d)]$ \\
        \bottomrule
    \end{tabular}
    \vspace{-0.5em}
\end{table*}

\clearpage
\section{Illustration of Differences Between  Adversarial vs VAUB Loss} \label{appendix:adv-vaub-compare}
In \autoref{fig:loss_compare}, we illustrate that using an adversarial loss in DANN does not provide a good measure of test performance while our VAUB-based loss has a strong correlation with test error.

\begin{figure}[ht!]
\minipage{0.45\textwidth}
\centering
  \includegraphics[width=\linewidth]{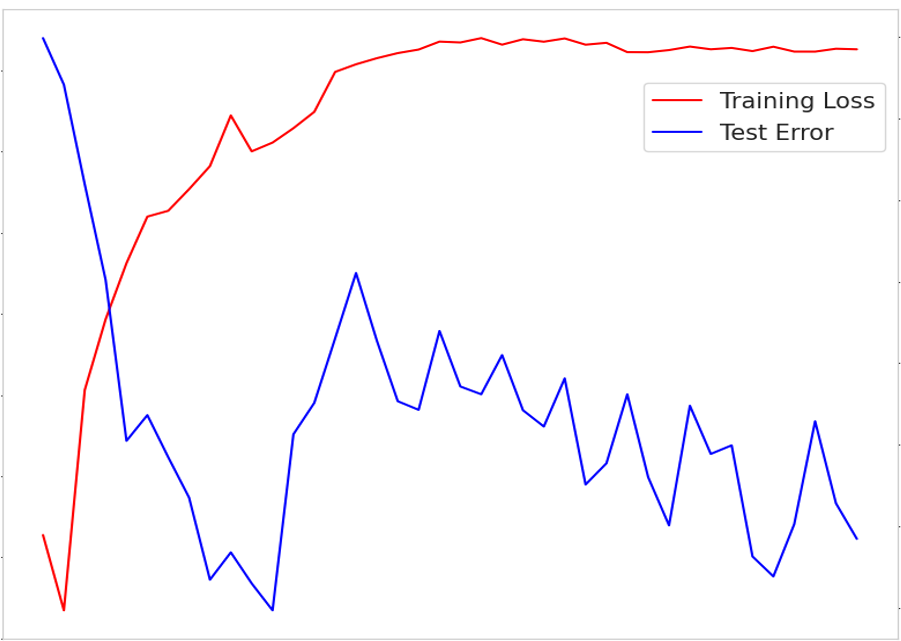}
  \caption*{(a) DANN loss vs test accuracy}
\endminipage \hspace{1.5em}
\minipage{0.45\textwidth}
\centering
  \includegraphics[width=\linewidth]{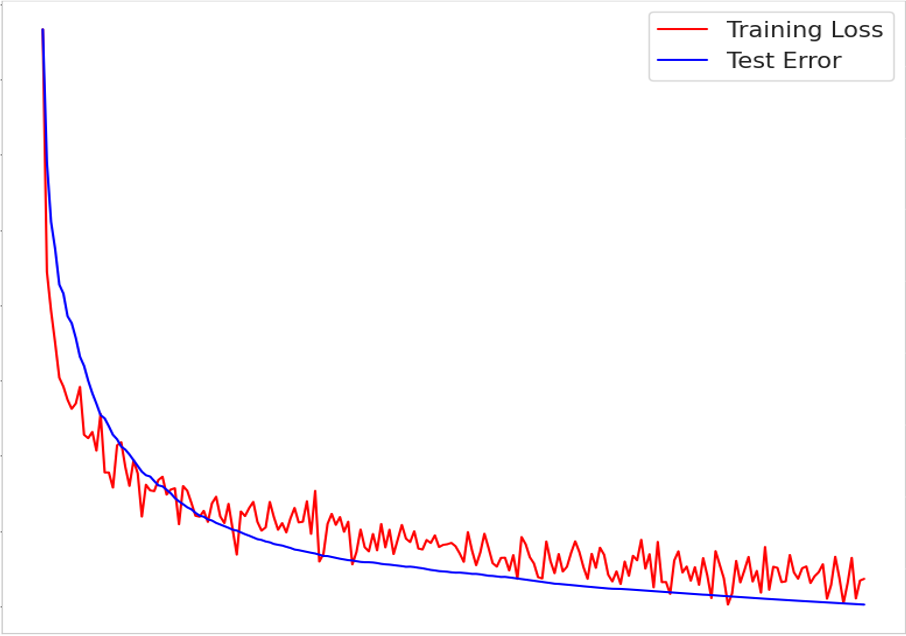}
  \caption*{(b) VAUB loss vs test accuracy}
\endminipage \\
\caption{This figure illustrates the tendency of test error and training loss for both models. For both figures, $x$-axis represents the epochs while the $y$-axis indicates the value of corresponding metric.}
\label{fig:loss_compare}
\end{figure}

\section{Proofs} \label{appendix:proof}
\subsection{Proof that JSD is Invariant Under Invertible Transformations} \label{appendix:proof-JSD-invar}
See proof from Theorem 1 of \citet{tran2021data}.

\subsection{Proof of Entropy Change of Variables For Probabilistic Autoencoders}
An entropy change of variables bound inspired by the derivations of surjective VAE flows \cite{nielsen2020survae} can be derived so that we can apply a similar proof as in \citet{cho2022cooperative}.
\begin{lemma}[Entropy Change of Variables for Probabilistic Autoencoders] 
\label{thm:probabilistic-entropy-change-of-variables}
Given an encoder $q(\zvec|\xvec)$ and a variational decoder $p(\xvec|\zvec)$, the latent entropy can be lower bounded as follows:
\begin{align}
    \HH(q(\zvec)) = \HH(q(\xvec)) 
    + \E_{q(\xvec,\zvec)} \left [\log \frac{p(\xvec|\zvec)}{q(\zvec |\xvec)}\right] +\E_{q(\zvec)}[\KL(q(\xvec|\zvec),p(\xvec|\zvec)] \geq \HH(q(\xvec)) 
    + \E_{q(\xvec,\zvec)} \left [\log \frac{p(\xvec|\zvec)}{q(\zvec |\xvec)}\right] \,,
\end{align}
where the bound gap is exactly $\E_{q(\zvec)}[\KL(q(\xvec|\zvec),p(\xvec|\zvec)]$, which can be made tight if the right hand side is maximized w.r.t. $p(\xvec|\zvec)$. 
\end{lemma}

\begin{proof}

Similar to the bounds for ELBO, we inflate by both the encoder $q(z|x)$ and the decoder $p(x|z)$ and eventually bring out an expectation over a KL term.
\begin{align}
\HH(q(x)) &= \E_{q(x)}[-\log q(x)] \\
&= \E_{q(x)q(z|x)}[-\log q(x)] \\
&= \E_{q(x)q(z|x)}[\log q(z|x) - \log q(z|x)q(x) ] \\
&= \E_{q(x)q(z|x)}[\log q(z|x) - \log q(z)q(x|z) ] \\
&= \E_{q(x)q(z|x)}[-\log q(z) + \log q(z|x) - \log q(x|z) ] \\
&= \HH(q(z)) + \E_{q(x)q(z|x)}\left[\log \frac{q(z|x)}{q(x|z)}\right]  \\
&= \HH(q(z)) + \E_{q(x)q(z|x)}\left[\log \frac{q(z|x)p(x|z)}{p(x|z)q(x|z)}\right]  \\
&= \HH(q(z)) + \E_{q(x)q(z|x)}\left[\log \frac{q(z|x)}{p(x|z)}\right] + \E_{q(x)q(z|x)}\left[\log \frac{p(x|z)}{q(x|z)}\right] \\
&= \HH(q(z)) + \E_{q(x)q(z|x)}\left[\log \frac{q(z|x)}{p(x|z)}\right] - \E_{q(z)}\left[\E_{q(x|z)}\left[\log \frac{q(x|z)}{p(x|z)}\right]\right]\\
&= \HH(q(z)) + \E_{q(x)q(z|x)}\left[\log \frac{q(z|x)}{p(x|z)}\right] - \E_{q(z)}[\KL(q(x|z),p(x|z))] \,.
\end{align}
By rearranging the above equation, we can easily derive the result from the non-negativity of KL:
\begin{align}
    \HH(q(z)) 
    &= \HH(q(x)) - \E_{q(x)q(z|x)}\left[\log \frac{q(z|x)}{p(x|z)}\right]  +\E_{q(z)}[\KL(q(x|z),p(x|z))]  \\
    &= \HH(q(x)) + \E_{q(x)q(z|x)}\left[\log \frac{p(x|z)}{q(z|x)}\right]  +\E_{q(z)}[\KL(q(x|z),p(x|z))] \\
    &\geq \HH(q(x)) + \E_{q(x)q(z|x)}\left[\log \frac{p(x|z)}{q(z|x)}\right] \,,
\end{align}
where it is clear that the bound gap is exactly $\E_{q(z)}[\KL(q(x|z),p(x|z))]$.
Furthermore, we note that by maximizing over all possible $p(x|z)$ (or minimizing the negative objective), we can make the bound tight:
\begin{align}
    &\argmax_{p(x|z)} \HH(q(x)) + \E_{q(x)q(z|x)}\left[\log \frac{p(x|z)}{q(z|x)}\right] \\
    &=\argmax_{p(x|z)} \E_{q(x)q(z|x)}[\log p(x|z)] \\
    &=\argmin_{p(x|z)} \E_{q(x)q(z|x)}[-\log p(x|z)] \\
    &=\argmin_{p(x|z)} \E_{q(x)q(z|x)}[-\log p(x|z) + \log q(x|z)] \\
    &=\argmin_{p(x|z)} \E_{q(z)}[\KL(q(x|z), p(x|z))] \,,
\end{align}
where the minimum is clearly when $p(x|z) = q(x|z)$ and the KL terms become 0.
Thus, the gap can be reduced by maximizing the right hand side w.r.t. $p(x|z)$ and can be made tight if $p(x|z) = q(x|z)$.
\end{proof}

\subsection{Proof of \autoref{thm:vaub-upper-bound} (VAUB is an Upper Bound on GJSD)}
\begin{proof}
    The proof is straightforward using \autoref{thm:gjsd-upper-bound} from \citet{cho2022cooperative} and \autoref{thm:probabilistic-entropy-change-of-variables} applied to each domain-conditional distribution $q(\zvec|\d)$:
    \begin{align}
        &\GJSD(\{q(\zvec|\d)\}_{\d=1}^\ndist) \\
        &\leq \min_{p(\zvec)} \Hc(q(\zvec), p(\zvec)) -\E_{q(\d)}[\HH(q(\zvec|\d))] \tag{\autoref{thm:gjsd-upper-bound}} \\ 
        &\leq \min_{\substack{p(\zvec) \\ p(\xvec|\zvec,\d)}} \Hc(q(\zvec), p(\zvec)) -\E_{q(\d)}\left[\HH(q(\xvec|\d)) + \E_{q(\xvec|\d)q(\zvec|\xvec,\d)}\left[\log \frac{p(\xvec|\zvec,\d)}{q(\zvec|\xvec,\d)}\right]\right] \tag{\autoref{thm:probabilistic-entropy-change-of-variables}} \\ 
        &= \min_{\substack{p(\zvec) \\ p(\xvec|\zvec,\d)}} \E_q[-\log p(\zvec)] -\E_{q(\d)}[\HH(q(\xvec|\d))] + \E_{q}\left[-\log \frac{p(\xvec|\zvec,\d)}{q(\zvec|\xvec,\d)}\right] \\ 
        &= \min_{\substack{p(\zvec) \\ p(\xvec|\zvec,\d)}} \E_q\left[-\log \left(\frac{p(\xvec|\zvec,\d)}{q(\zvec|\xvec,\d)}\cdot p(\zvec)\right)\right] -\E_{q(\d)}[\HH(q(\xvec|\d))] \\ 
        &\triangleq \mathrm{VAUB}(q(\zvec|\xvec,\d)) \,,
    \end{align}
    where the last two equals are just by definition of cross entropy and rearrangement of terms.
    From \autoref{thm:gjsd-upper-bound} and \autoref{thm:probabilistic-entropy-change-of-variables}, we know that the bound gaps for both inequalities is:
    \begin{align}
        \KL(q(\zvec), p(\zvec)) + \E_{q(\d)q(\zvec|\d)}[\KL(q(\xvec|\zvec,\d),p(\xvec|\zvec,\d))] \,,
    \end{align}
    where both KL terms can be made 0 by minimizing the VAUB over all possible $p(\zvec)$ and $p(\xvec|\zvec,\d)$ respectively.
\end{proof}

\subsection{Proof of \autoref{thm:log-determinant} (Ratio of Decoder and Encoder Term Interpretation)}
\begin{proof}
Given that the variational optimization is solved perfectly so that $p^*(\xvec|\zvec,\d) = q(\xvec|\zvec,\d)$, we can notice that this ratio has a simple form as the ratio of marginal densities:
\begin{align}
    \E_{q(\zvec|\xvec,\d)}\Big[-\log \frac{p^*(\xvec|\zvec,\d)}{q(\zvec|\xvec,\d)}\Big] 
    = \E_q\Big[-\log \frac{q(\xvec|\zvec,\d)}{q(\zvec|\xvec,\d)} \Big]
    = \E_q\Big[-\log\frac{q(\xvec,\zvec|\d)}{q(\zvec|\d)}\frac{q(\xvec|\d)}{q(\xvec,\zvec|\d)}\Big]
    = \E_q\Big[-\log\frac{q(\xvec|\d)}{q(\zvec|\d)}\Big]\,,
\end{align}
where the first equals is just by assumption that $p^*(\xvec|\zvec,\d)$ is optimized perfectly.
Now if the encoder is an invertible and deterministic function, i.e., $q(\zvec|\xvec,\d)$ is a Dirac delta at $g(\xvec|\d)$, then we can derive that the marginal ratio is simply the Jacobian in this special case: 
\begin{align}
    \E_{q(\zvec|\xvec,\d)}\Big[-\log\frac{q(\xvec|\d)}{q(\zvec|\d)}\Big]
    =-\log\frac{q(\xvec|\d)}{q(\zvec=g(\xvec|\d)|\d)}
    =-\log\frac{|J_g(\xvec|\d)| q(\zvec=g(\xvec|\d)|\d)}{q(\zvec=g(\xvec|\d)|\d)}
    =-\log |J_g(\xvec|\d)| \,,
\end{align}
where the first equals is because the encoder $q(\zvec|\xvec,\d)$ is a Dirac delta function such that $\zvec = g(\xvec|\d)$, and the second equals is by the change of variables formula.
\end{proof}

\subsection{Proof that Mutual Information is Bounded by Reconstruction Term}

We include a simple proof that the mutual information can be bounded by a probabilistic reconstruction term.
\begin{proof}
For any $\tilde{p}(\xvec|\zvec,\d)$, we know the following:
\begin{align}
    I(\xvec, \zvec|\d)
    &=\HH(\xvec|\d)-\HH(\xvec|\zvec, \d)  \\
    &=\HH(\xvec|\d) -\E_q[-\log q(\xvec|\zvec,\d)]  \\
    &=\HH(\xvec|\d) +\E_q\big[\log \frac{q(\xvec|\zvec,\d)\tilde{p}(\xvec|\zvec,\d)}{\tilde{p}(\xvec|\zvec,\d)}\big]  \\
    &=\E_{q(\xvec|\zvec)}[\log \tilde{p}(\xvec|\zvec,\d)] + \KL(q(\xvec|\zvec,\d), \tilde{p}(\xvec|\zvec,\d)) + \HH(\xvec|\d) \\
    &\geq\E_{q(\xvec|\zvec)}[\log \tilde{p}(\xvec|\zvec,\d)] + C\,,
\end{align}
where $C\triangleq \HH(\xvec|\d)$, which is constant in the optimization.
Therefore, we can optimize over all $\tilde{p}$ and still get a lower bound on mutual information:
\begin{align}
    I(\xvec, \zvec|\d) \geq \max_{\tilde{p}(\xvec|\zvec,\d)} \E_{q(\xvec|\zvec)}[\log \tilde{p}(\xvec|\zvec,\d)] + C
\end{align}
where $C$ in constant w.r.t. to the parameters of interest.
\end{proof}
%
%

\subsection{Proof of \autoref{thm:nsj-is-divergence} (Noisy JSD is a Statistical Divergence)}

We prove a slightly more general version of noisy JSD here where the added Gaussian noise can come from a distribution over noise levels.
While the Noisy JSD definition uses a single noise value, this can be generalized to an expectation over different noise scales as in the next definition.

\begin{definition}[Noised-Smoothed JSD]
Given a distribution over noise variances $p_\sigma$ that has support on the positive real numbers, the noise-smoothed JSD (NSJ) is defined as:
\begin{align}
    \mathrm{NSJ}(p,q) 
    &= \E_{\sigma}[\mathrm{NJSD}_{\sigma}(p, q)]
    = \E_{\sigma}[\mathrm{JSD}(\tilde{p}_\sigma, \tilde{q}_\sigma) ] \,,
    \label{eqn:njs-as-jsd}
\end{align}
where $\tilde{p}_{\sigma} \triangleq p \ast \mathcal{N}(0, \sigma^2 I)$ and similarly for $\tilde{q}_{\sigma}$.
\end{definition}
Note that NJSD is a special case of NSJ where $p_{\sigma}$ is a Dirac delta distribution at a single $\sigma$ value.
Now we give the proof that NSJ (and thus NJSD) is a statistical divergence.

\begin{proof}
    NSJ is non-negative because Eqn.6 (in main paper) is merely an expectation over JSDs, which are non-negative by the property of JSD.
    Now we prove the identity property for divergences, i.e., that $NSJ(p,q) = 0 \Leftrightarrow p = q$.
    If $p=q$, then it is simple to see that all the inner JSD terms will be 0 and thus $NSJ(p,q)=0$. 
    For the other direction, we note that if $NSJ(p,q) = 0$, we know that every NJSD term in the expectation in Eqn.6 (in main paper) must be 0, i.e., $\forall \sigma \in \textnormal{supp}(p_\sigma), NJSD(p,q)=0$.
    Thus, we only need to prove for NJSD. 
    For NJSD, we note that convolution with a Gaussian kernel $k \triangleq \mathcal{N}(0, \sigma^2 I)$ is invertible, and thus:
    \begin{align}
        &NJSD(p,q) = 0 \Rightarrow JSD(p \ast k, q \ast k) = 0 \\
        &\Rightarrow p \ast k = q \ast k \Rightarrow p = q \,.
    \end{align} 
\end{proof}

\subsection{Proof of \autoref{thm:noisy-upper-bounds} (Noisy AUB and Noisy VAUB Upper Bounds)}

We would like to show that the noisy JSD can be upper bounded by a noisy version of AUB.
Again, the key here is considering the latent entropy terms. So we provide one lemma and a corollary to setup the main proof.
\begin{lemma}[Noisy entropy inequality]
\label{thm:noisy-entropy-inequality}
The entropy of a noisy random variable is greater than the entropy of its clean counterpart, i.e., if $\tilde{z}\triangleq z + \epsilon \sim q(\tilde{z})$ where $z\sim q(z)$ and $\epsilon$ are independent random variables, then $\HH(q(\tilde{z})) \geq \HH(q(z))$. (Proof are provided in the appendix)
\end{lemma}
\begin{proof}
\begin{align}
    \HH(z + \epsilon) 
    &\geq \HH(z + \epsilon|\epsilon) \tag{Conditioning reduces entropy} \\
    &=\HH(z|\epsilon) \tag{Entropy is invariant under constant shift}\\
    &=\HH(z)\,. \tag{Independence of $z$ and $\epsilon$}
\end{align}
\end{proof}

\begin{corollary}[Noisy entropy inequalities]
Given a noisy random variable $\tilde{z} \triangleq z + \epsilon \sim q(\tilde{z})$,  
the following inequalities hold for deterministic invertible mappings $g$ and stochastic mappings $q(z|x)$, respectively:
\begin{align}
    \HH(q(\tilde{z})) &\geq \HH(q(z)) \geq \HH(q(x)) + \E_{q(x)}[\log |J_g(x)|] \\ 
    \HH(q(\tilde{z})) &\geq \HH(q(z)) \geq \HH(q(x)) + \E_{q(x)q(z|x)}\left[\log \frac{p(x|z)}{q(z|x)}\right] \,.
\end{align}
\end{corollary}
\begin{proof}
    Inequalities follow directly from \autoref{thm:noisy-entropy-inequality} and the entropy inequalities in \autoref{thm:entropy-change-of-variables} and \autoref{thm:probabilistic-entropy-change-of-variables} respectively.
\end{proof}

Given these entropy inequalities, we now provide the proof that NAUB and NVAUB are upper bounds of the noisy JSD counterparts using the same techniques as \citep{cho2022cooperative} and the proof for VAUB above.

\begin{proof}
Proof of upper bound for noisy version of flow-based AUB:
\begin{align}
    \textnormal{NGJSD}&(\{q(\zvec|\d)\}_{\d=1}^\ndist; \sigma^2) \\
    &\equiv \GJSD(\{q(\tilde{\zvec}|\d)\}_{\d=1}^\ndist) \\
    &\leq \min_{p(\tilde{\zvec}) \in \pset_{\tilde{\zvec}}} \Hc(q(\tilde{\zvec}), p(\tilde{\zvec})) -\E_{q(\d)}[\HH(q(\tilde{\zvec}|\d))] \\ 
    &\leq \min_{p(\tilde{\zvec}) \in \pset_{\tilde{\zvec}}} \Hc(q(\tilde{\zvec}), p(\tilde{\zvec})) -\E_{q(\d)}[\E_{q(x|d)}[\log |J_g(\xvec|d)|]  + \HH(q(\xvec|\d))] \\
    &= \min_{p(\tilde{\zvec}) \in \pset_{\tilde{\zvec}}} \E_{q(\tilde{\zvec})}[-\log p(\tilde{\zvec})] -\E_{q(\xvec,\d)}[\log |J_g(\xvec|d)|]  - \E_{q(d)}[\HH(q(\xvec|\d))] \\
    &= \min_{p(\tilde{\zvec}) \in \pset_{\tilde{\zvec}}} \E_{q(\xvec,d)q(\epsilon;\sigma^2)}[-\log |J_g(\xvec|d)| p(g(\xvec|d) + \epsilon)] - \E_{q(d)}[\HH(q(\xvec|\d))] \\
    &\triangleq \textnormal{NAUB}(q(\zvec|\xvec,\d); \sigma^2) \,.
\end{align}

Proof of upper bound for noisy version of VAUB:
\begin{align}
    \textnormal{NGJSD}&(\{q(\zvec|\d)\}_{\d=1}^\ndist; \sigma^2) \\
    &\equiv \GJSD(\{q(\tilde{\zvec}|\d)\}_{\d=1}^\ndist) \\
    &\leq \min_{p(\tilde{\zvec}) \in \pset_{\tilde{\zvec}}} \Hc(q(\tilde{\zvec}), p(\tilde{\zvec})) -\E_{q(\d)}[\HH(q(\tilde{\zvec}|\d))] \\ 
    &\leq \min_{p(\tilde{\zvec}) \in \pset_{\tilde{\zvec}}} \Hc(q(\tilde{\zvec}), p(\tilde{\zvec})) -\E_{q(\d)} \left[\max_{p(\xvec|\zvec,\d) \in \pset_{\xvec|\zvec,\d}} 
    \E_{q(\xvec,\zvec|\d)} \left [\log \frac{p(\xvec|\zvec, \d)}{q(\zvec |\xvec, \d)}\right] + \HH(q(\xvec|\d)) \right] \\
    &= \min_{\substack{p(\tilde{\zvec}) \in \pset_{\tilde{\zvec}} \\
    p(\xvec|\zvec,\d) \in \pset_{\xvec|\zvec,\d}}} 
    \E_{q(\tilde{\zvec})}[-\log p(\tilde{\zvec})] - \E_{q(\d)} \left[\E_{q(\xvec,\zvec|\d)}\left[\log \frac{p(\xvec|\zvec, \d)}{q(\zvec |\xvec, \d)}\right] + \HH(q(\xvec|\d)) \right] \\
    &= \min_{\substack{p(\tilde{\zvec}) \in \pset_{\tilde{\zvec}} \\ 
    p(\xvec|\zvec,\d) \in \pset_{\xvec|\zvec,\d}}} 
    \E_{q(\xvec,\zvec,\d,\tilde{\zvec})}\left[-\log \left(\frac{p(\xvec|\zvec, \d)}{q(\zvec |\xvec, \d)}\cdot p(\tilde{\zvec})\right)\right] - \E_{q(\d)}[ \HH(q(\xvec|\d))] \\
    &= \min_{\substack{p(\tilde{\zvec}) \in \pset_{\tilde{\zvec}} \\ 
    p(\xvec|\zvec,\d) \in \pset_{\xvec|\zvec,\d}}} 
    \E_{q(\xvec,\zvec,\d)q(\epsilon;\sigma^2)}\left[-\log \left(\frac{p(\xvec|\zvec, \d)}{q(\zvec |\xvec, \d)}\cdot p(\zvec + \epsilon)\right)\right] - \E_{q(\d)}[ \HH(q(\xvec|\d))] \\
    &\triangleq \textnormal{NVAUB}(q(\zvec|\xvec,\d); \sigma^2) \,.
\end{align}
\end{proof}

\subsection{Proof of \autoref{thm:fixed-prior} (Fixed Prior is VAUB Plus Regularization Term)}
We first remember the Fair VAE objective \citep{louizos2015variational} (note $q$ is encoder distribution and $p$ is decoder distribution and $q(\zvec)$ is the marginal distribution of $q(\xvec,d,\zvec):=q(\xvec,d)q(\zvec|\xvec,d)$):
\begin{align}
    &\min_{q(\zvec|\xvec,d)} \min_{p(\xvec|\zvec,d)} \E_{q}\big[-\log \frac{p(\xvec|\zvec,d)}{q(\zvec|\xvec,d)} p_{\mathcal{N}(0,I)}(\zvec) \big] 
\end{align}
We show that this objective can be decomposed into an alignment objective and a prior regularization term if we assume the optimization class of prior distributions from the VAUB includes all possible distributions (and is solved theoretically).
This gives a precise characterization of how the Fair VAE bound w.r.t. alignment and compares it to the VAUB alignment objective.

This decomposition exposes two insights.
First, the Fair VAE objective is an alignment loss plus a regularization term.
Thus, the \emph{Fair VAE objective is sufficient for alignment but not necessary}---it adds an additional constraint/regularization that is not necessary for alignment.
Second, it reveals that by fixing the prior distribution, it can be viewed as perfectly solving the optimization over the prior for the alignment objective but requiring an unnecessary prior regularization.
Finally, it should be noted that it is not possible in practice to compute the prior regularization term because $q(\zvec)$ is not known.
Therefore, this decomposition is only useful to understand the structure of he objective theoretically.

\begin{proof}[Proof of \autoref{thm:fixed-prior}]
\begin{align}
    &\min_{q(\zvec|\xvec,d)} \min_{p(\xvec|\zvec,d)} \E_{q}\big[-\log \frac{p(\xvec|\zvec,d)}{q(\zvec|\xvec,d)} p_{\mathcal{N}(0,I)}(\zvec) \big] \tag{Fair VAE}\\
    &=\min_{q(\zvec|\xvec,d)} \min_{p(\xvec|\zvec,d)} \E_{q}\big[-\log \frac{p(\xvec|\zvec,d)}{q(\zvec|\xvec,d)} \frac{p_{\mathcal{N}(0,I)}(\zvec) q(\zvec)}{q(\zvec)} \big] \tag{Inflate with true marginal $q(\zvec)$} \\
    &=\min_{q(\zvec|\xvec,d)} \min_{p(\xvec|\zvec,d)} \E_{q}\big[-\log \frac{p(\xvec|\zvec,d)}{q(\zvec|\xvec,d)} q(\zvec) \big] + \E_{q}\big[ -\log \frac{p_{\mathcal{N}(0,I)}(\zvec)}{q(\zvec)} \big] \tag{Rearrange} \\
    &=\min_{q(\zvec|\xvec,d)} \min_{p(\xvec|\zvec,d)} \E_{q}\big[-\log \frac{p(\xvec|\zvec,d)}{q(\zvec|\xvec,d)} q(\zvec) \big] + \E_{q}\big[ \log \frac{q(\zvec)}{p_{\mathcal{N}(0,I)}(\zvec)} \big] \tag{Push negative inside}\\
    &=\min_{q(\zvec|\xvec,d)} \underbrace{\min_{p(\xvec|\zvec,d)} \E_{q}\big[-\log \frac{p(\xvec|\zvec,d)}{q(\zvec|\xvec,d)} q(\zvec) \big]}_{\text{VAUB Alignment with perfect prior optimization}} + \underbrace{\KL(q(\zvec),p_{\mathcal{N}(0,I)}(\zvec))}_{\text{Prior regularization}} \tag{Definition of KL} \\
    &=\min_{q(\zvec|\xvec,d)} \min_{p(\xvec|\zvec,d)} \Big( \min_{p(\zvec)} \E_{q}\big[-\log \frac{p(\xvec|\zvec,d)}{q(\zvec|\xvec,d)} p(\zvec) \big] \Big) + \KL(q(\zvec),p_{\mathcal{N}(0,I)}(\zvec)) \tag{Replace $q(\zvec)$ with optimization over $p(\zvec)$}  \\
    &=\min_{q(\zvec|\xvec,d)} \underbrace{\Big(\min_{p(\xvec|\zvec,d)} \min_{p(\zvec)} \E_{q}\big[-\log \frac{p(\xvec|\zvec,d)}{q(\zvec|\xvec,d)} p(\zvec) \big] \Big)}_{\text{VAUB Alignment Objective}} + \underbrace{\KL(q(\zvec),p_{\mathcal{N}(0,I)}(\zvec))}_{\text{Prior Regularization}}  \tag{Regroup to show structure}
\end{align}
The last line is by noticing that $\KL(q(\zvec),p_{\mathcal{N}(0,I)}(\zvec))$ does not depend on $p(\xvec|\zvec,d)$ or $p(\zvec)$, i.e., it only depends on $q(\zvec|\xvec,d)$ and the original data distribution $q(\xvec,d)$.
\end{proof}

\section{Experimental Setup} \label{appendix:experiment-setup}
All experiments were conducted on a computing setup with 24 processors, each having 12 cores running at 3.5GHz. Additionally, 2 NVIDIA RTX 3090 graphics cards were utilized when needed.
\subsection{Simulated Experiments}
\subsubsection{Non-Matching Dimensions between Latent Space and Input Space}
\paragraph{Dataset:} 
We have two datasets, namely $X_1$ and $X_2$, each consisting of 500 samples.
$X_1$ represents the original moon dataset, which has been perturbed by adding a noise scale of 0.05. 
$X_2$ is created by applying a transformation to the moon dataset. First, a rotation matrix of $\frac{3\pi}{8}$ is applied to the moons dataset which is generated using the same noise scale as in $X_1$. Then, scaling factors of 0.75 and 1.25 are applied independently to each dimension of the dataset. This results in a rotated-scaled version of the original moon dataset distribution.

\paragraph{Model:} Encoders consist of three fully connected layers with hidden layer size as $20$. Decoders are the reverse setup of the encoders. $P_z$ is a learnable one-dimensional mixture of Gaussian distribution with $10$ components and diagonal covariance matrix.

\subsubsection{Noisy-AUB Helps Mitigate the Vanishing Gradient Problem}
\paragraph{Dataset:} $X_1$ Gaussian distribution with mean $-20$ and unit variance, $X_2$ Gaussian distribution with mean $20$ and unit variance. Each dataset has $500$ samples.

\paragraph{Model:} Encoders consist of three fully connected layers with hidden layer size as $10$. Decoders are the reverse setup of the encoders. $P_z$ is a learnable one-dimensional mixture of Gaussian distribution with $2$ components and diagonal covariance matrix. The NVAUB has the added noise level of $10$ while the VAUB has no added noise. 

\subsection{Comparison Between Other Non-adversarial Bounds}
\paragraph{Dataset:} We adopted the preprocessed Adult Dataset from \cite{Zhao2020Conditional}, where the processed data has input dimensions $114$ , targeted attribute as \emph{income} and sensitive attribute as \emph{gender}. 

\paragraph{Model:} Since all baseline models have only one encoder, we also adapt our model to have shared encoders. All models have encoder consists of three fully connected layers with hidden layer size as $84$ and latent features as $60$. 
For \cite{moyer2018invariant} and ours, decoders are the reverse setup of the encoder. For \cite{gupta2021controllable}, we adapt the same network setup for the contrastive estimation model. 
Again, for \cite{moyer2018invariant} and \cite{gupta2021controllable} we used a fixed Gaussian distribution, and for our model, we use a learnable mixture of Gaussian distribution with $5$ components and diagonal covariance matrix.
For this experiment, we manually delete the classifier loss in all baseline models for the purpose of comparing only the bound performances.

\paragraph{Metric:} 
For SWD, we randomly project $10^3$ directions to one-dimensional vectors and compute the 1-Wasserstein distance between the projected vectors.
Here is the table for the corresponding mean and standard deviation.
The models are all significantly different (i.e., $p$-value is less than 0.01) when using an unpaired $t$-test on the $10^3$ SWD values for each method.

\begin{table}[ht]
\centering
\caption{SWD for each method where * denotes statistically different at a 99\% confidence level.}
\begin{tabular}{lccc}
\toprule
 & Moyer & Gupta & VAUB \\
\midrule
Sample Mean & 9.71* & 6.7* & 5.64* \\
$\sigma$ & 0.54 & 0.74 & 0.64 \\
\bottomrule
\end{tabular}
\end{table}

For SVM, we first split the test dataset with $80\%$ for training the SVM model and $20\%$ for evaluating the SVM model. We use the scikit-learn package to grid search over the logspace of the $C$ and $\gamma$ parameters to choose the best hyperparameters in terms of accuracy.

\subsection{Replacing Adversarial Losses}
\subsubsection{Replacing the Domain Adaption Objective}
\paragraph{Model:}
We use the same encoder setup (referred as feature extraction layers) and the same classifier structure as in \cite{ganin2016domain}).

\subsubsection{Replacing the Fairness Representation Objective}
\paragraph{Dataset:} We adopted the preprocessed Adult Dataset from \cite{Zhao2020Conditional}, where the processed data has input dimensions $114$ , targeted attribute as \emph{income} and sensitive attribute as \emph{gender}. 

\paragraph{Model:} Since all baseline models have only one encoder, we also adapt our model to have a shared encoder. All models have encoder consists of three fully connected layers with hidden layer size as $84$ and latent features as $60$. 
For \cite{moyer2018invariant} and ours, decoders are the reverse setup of the encoder. For \cite{gupta2021controllable}, we adapt the same network setup for the contrastive estimation model. 
Again, for \cite{moyer2018invariant} and \cite{gupta2021controllable} we used a fixed Gaussian distribution, and for our model, we use a learnable mixture of Gaussian distribution with $5$ components and diagonal covariance matrix.

%% file: sample_paper.bbl
\begin{thebibliography}{}

\bibitem[Arjovsky and Bottou, 2017]{arjovsky2017towards}
Arjovsky, M. and Bottou, L. (2017).
\newblock Towards principled methods for training generative adversarial networks.
\newblock In {\em International Conference on Learning Representations}.

\bibitem[Arjovsky et~al., 2017]{arjovsky2017wasserstein}
Arjovsky, M., Chintala, S., and Bottou, L. (2017).
\newblock Wasserstein generative adversarial networks.
\newblock In {\em International conference on machine learning}, pages 214--223. PMLR.

\bibitem[Cho et~al., 2022]{cho2022cooperative}
Cho, W., Gong, Z., and Inouye, D.~I. (2022).
\newblock Cooperative distribution alignment via jsd upper bound.
\newblock In {\em Neural Information Processing Systems (NeurIPS)}.

\bibitem[Creager et~al., 2019]{creager2019flexibly}
Creager, E., Madras, D., Jacobsen, J.-H., Weis, M., Swersky, K., Pitassi, T., and Zemel, R. (2019).
\newblock Flexibly fair representation learning by disentanglement.
\newblock In {\em International conference on machine learning}, pages 1436--1445. PMLR.

\bibitem[Farnia and Ozdaglar, 2020]{pmlr-v119-farnia20a}
Farnia, F. and Ozdaglar, A. (2020).
\newblock Do {GAN}s always have {N}ash equilibria?
\newblock In III, H.~D. and Singh, A., editors, {\em Proceedings of the 37th International Conference on Machine Learning}, volume 119 of {\em Proceedings of Machine Learning Research}, pages 3029--3039. PMLR.

\bibitem[Ganin et~al., 2016]{ganin2016domain}
Ganin, Y., Ustinova, E., Ajakan, H., Germain, P., Larochelle, H., Laviolette, F., Marchand, M., and Lempitsky, V. (2016).
\newblock Domain-adversarial training of neural networks.
\newblock {\em The journal of machine learning research}, 17(1):2096--2030.

\bibitem[Goodfellow et~al., 2014]{goodfellow2014generative}
Goodfellow, I., Pouget-Abadie, J., Mirza, M., Xu, B., Warde-Farley, D., Ozair, S., Courville, A., and Bengio, Y. (2014).
\newblock Generative adversarial nets.
\newblock {\em Advances in neural information processing systems}, 27.

\bibitem[Grover et~al., 2020]{grover2020alignflow}
Grover, A., Chute, C., Shu, R., Cao, Z., and Ermon, S. (2020).
\newblock Alignflow: Cycle consistent learning from multiple domains via normalizing flows.
\newblock In {\em AAAI}.

\bibitem[Gupta et~al., 2021]{gupta2021controllable}
Gupta, U., Ferber, A.~M., Dilkina, B., and Ver~Steeg, G. (2021).
\newblock Controllable guarantees for fair outcomes via contrastive information estimation.
\newblock In {\em Proceedings of the AAAI Conference on Artificial Intelligence}, volume~35, pages 7610--7619.

\bibitem[Han et~al., 2023]{han2023ffb}
Han, X., Chi, J., Chen, Y., Wang, Q., Zhao, H., Zou, N., and Hu, X. (2023).
\newblock {FFB: A Fair Fairness Benchmark for In-Processing Group Fairness Methods}.

\bibitem[Higgins et~al., 2017]{higgins2017betavae}
Higgins, I., Matthey, L., Pal, A., Burgess, C., Glorot, X., Botvinick, M., Mohamed, S., and Lerchner, A. (2017).
\newblock beta-{VAE}: Learning basic visual concepts with a constrained variational framework.
\newblock In {\em International Conference on Learning Representations}.

\bibitem[Kingma et~al., 2019]{kingma2019introduction}
Kingma, D.~P., Welling, M., et~al. (2019).
\newblock An introduction to variational autoencoders.
\newblock {\em Foundations and Trends{\textregistered} in Machine Learning}, 12(4):307--392.

\bibitem[Kurach et~al., 2019]{kurach2019the}
Kurach, K., Lucic, M., Zhai, X., Michalski, M., and Gelly, S. (2019).
\newblock The {GAN} landscape: Losses, architectures, regularization, and normalization.

\bibitem[LeCun et~al., 2010]{lecun2010mnist}
LeCun, Y., Cortes, C., and Burges, C. (2010).
\newblock Mnist handwritten digit database.
\newblock {\em ATT Labs [Online]. Available: http://yann.lecun.com/exdb/mnist}, 2.

\bibitem[Lin, 1991]{lin1991divergence}
Lin, J. (1991).
\newblock Divergence measures based on the shannon entropy.
\newblock {\em IEEE Transactions on Information theory}, 37(1):145--151.

\bibitem[Louizos et~al., 2015]{louizos2015variational}
Louizos, C., Swersky, K., Li, Y., Welling, M., and Zemel, R. (2015).
\newblock The variational fair autoencoder.
\newblock {\em arXiv preprint arXiv:1511.00830}.

\bibitem[Lucic et~al., 2018]{lucic2018gans}
Lucic, M., Kurach, K., Michalski, M., Gelly, S., and Bousquet, O. (2018).
\newblock Are gans created equal? a large-scale study.
\newblock In {\em Advances in neural information processing systems}, pages 700--709.

\bibitem[Madras et~al., 2018]{madras2018learning}
Madras, D., Creager, E., Pitassi, T., and Zemel, R. (2018).
\newblock Learning adversarially fair and transferable representations.
\newblock In {\em International Conference on Machine Learning}, pages 3384--3393. PMLR.

\bibitem[Moyer et~al., 2018]{moyer2018invariant}
Moyer, D., Gao, S., Brekelmans, R., Galstyan, A., and Ver~Steeg, G. (2018).
\newblock Invariant representations without adversarial training.
\newblock {\em Advances in Neural Information Processing Systems}, 31.

\bibitem[Nie and Patel, 2020]{nie2020towards}
Nie, W. and Patel, A.~B. (2020).
\newblock Towards a better understanding and regularization of gan training dynamics.
\newblock In {\em Uncertainty in Artificial Intelligence}, pages 281--291. PMLR.

\bibitem[Nielsen et~al., 2020]{nielsen2020survae}
Nielsen, D., Jaini, P., Hoogeboom, E., Winther, O., and Welling, M. (2020).
\newblock Survae flows: Surjections to bridge the gap between vaes and flows.
\newblock {\em Advances in Neural Information Processing Systems}, 33:12685--12696.

\bibitem[Papamakarios et~al., 2021]{papamakarios2021normalizing}
Papamakarios, G., Nalisnick, E.~T., Rezende, D.~J., Mohamed, S., and Lakshminarayanan, B. (2021).
\newblock Normalizing flows for probabilistic modeling and inference.
\newblock {\em J. Mach. Learn. Res.}, 22(57):1--64.

\bibitem[Tran et~al., 2021]{tran2021data}
Tran, N.-T., Tran, V.-H., Nguyen, N.-B., Nguyen, T.-K., and Cheung, N.-M. (2021).
\newblock On data augmentation for gan training.
\newblock {\em IEEE Transactions on Image Processing}, 30:1882--1897.

\bibitem[Usman et~al., 2020]{usman2020log}
Usman, B., Sud, A., Dufour, N., and Saenko, K. (2020).
\newblock Log-likelihood ratio minimizing flows: Towards robust and quantifiable neural distribution alignment.
\newblock {\em Advances in Neural Information Processing Systems}, 33:21118--21129.

\bibitem[Wu et~al., 2020]{NEURIPS2020_3eb46aa5}
Wu, Y., Zhou, P., Wilson, A.~G., Xing, E., and Hu, Z. (2020).
\newblock Improving gan training with probability ratio clipping and sample reweighting.
\newblock In Larochelle, H., Ranzato, M., Hadsell, R., Balcan, M.~F., and Lin, H., editors, {\em Advances in Neural Information Processing Systems}, volume~33, pages 5729--5740. Curran Associates, Inc.

\bibitem[Xu et~al., 2018]{xu2018fairgan}
Xu, D., Yuan, S., Zhang, L., and Wu, X. (2018).
\newblock Fairgan: Fairness-aware generative adversarial networks.
\newblock In {\em 2018 IEEE International Conference on Big Data (Big Data)}, pages 570--575. IEEE.

\bibitem[Zhao et~al., 2020]{Zhao2020Conditional}
Zhao, H., Coston, A., Adel, T., and Gordon, G.~J. (2020).
\newblock Conditional learning of fair representations.
\newblock In {\em International Conference on Learning Representations}.

\bibitem[Zhao et~al., 2018]{zhao2018adversarial}
Zhao, H., Zhang, S., Wu, G., Moura, J.~M., Costeira, J.~P., and Gordon, G.~J. (2018).
\newblock Adversarial multiple source domain adaptation.
\newblock {\em Advances in neural information processing systems}, 31.

\end{thebibliography}
